\documentclass{article}

\usepackage{soul}

\usepackage{calculation}

\IfFileExists{iclr2026_conference.sty}{
  \usepackage{iclr2026_conference,times}
}{
  \usepackage{times}
  \usepackage[margin=1in]{geometry}
  \usepackage[numbers]{natbib}
  \newcommand{\iclrfinalcopy}{}
}

\usepackage{amsmath}

\DeclareMathOperator{\acos}{acos}

\usepackage{comment}
\usepackage{microtype}
\usepackage{graphicx}
\usepackage{subcaption}
\usepackage{booktabs}
\usepackage{amsmath,amssymb,amsthm}
\usepackage{mathtools}
\usepackage{url}
\usepackage{hyperref}
\usepackage{algorithm}
\usepackage{algpseudocode}

\usepackage{tikz}
\usetikzlibrary{angles,quotes}

\title{GSVD for Geometry-Grounded Dataset Comparison: An Alignment Angle Is All You Need}

 \newcommand{\Ker}{\mathrm{Ker}}

 \author{
   Arthur Sobrinho \\
   Institute of Computing, UFRJ \\
   \texttt{arthursfr@ic.ufrj.br}
   \And
   Eduarda Marques \\
  Institute of Computing, UFRJ \\
   \texttt{eduardasm@ic.ufrj.br}
   \And
   Jo\~ao Paix\~ao \\
   Institute of Computing, UFRJ \\
   \texttt{jpaixao@ic.ufrj.br}
   \And
   Daniel S. Menasche \\
   Institute of Computing, UFRJ \\
   \texttt{sadoc@ic.ufrj.br}
   \And
   Heudson Mirandola \\
   Institute of Mathematics, UFRJ \\
   \texttt{mirandola@im.ufrj.br}
 }
\iclrfinalcopy

\newcommand{\R}{\mathbb{R}}
\newcommand{\col}{\operatorname{col}}

\newcommand{\argmin}{\operatorname*{arg\,min}}
\newcommand{\argmax}{\operatorname*{arg\,max}}
\newcommand{\norm}[1]{\left\lVert #1 \right\rVert}

\newcommand{\safeincludegraphics}[2][]{%
  \IfFileExists{#2}{\includegraphics[#1]{#2}}{\fbox{\scriptsize Missing file: \texttt{\detokenize{#2}}}}%
}

\theoremstyle{definition}
\newtheorem{definition}{Definition}
\newtheorem{lemma}{Lemma}
\newtheorem{theorem}{Theorem}
\newtheorem{remark}{Remark}

\begin{document}

\maketitle

\begin{abstract}
Geometry-grounded learning asks models to respect structure in the problem domain rather than treating observations as arbitrary vectors. Motivated by this view, we revisit a classical but underused primitive for comparing datasets: \emph{linear relations} between two data matrices, expressed via the co-span constraint $Ax=By=z$ in a shared ambient space.

To operationalize this comparison, we use the generalized singular value decomposition (GSVD) as a joint coordinate system for two subspaces. In particular, we exploit the GSVD form
$A = H C U$, $B = H S V$ with $C^\top C + S^\top S = I$,
which separates shared versus dataset-specific directions through the diagonal structure of $(C,S)$.
From these factors we derive an interpretable \emph{angle score} $\theta(z)\in[0,\pi/2]$ for a sample $z$, quantifying whether $z$ is explained relatively more by $A$, more by $B$, or comparably by both. 

The primary role of $\theta(z)$ is as a \emph{per-sample geometric diagnostic}.  We illustrate the behavior of the score on MNIST through angle distributions and representative GSVD directions. A binary classifier derived from $\theta(z)$ is presented as an illustrative application of the score as an interpretable diagnostic tool.

\end{abstract}

\section{Introduction}
Comparing datasets is a recurring problem across machine learning and data analysis.
It arises when assessing dataset shift between training and deployment~\cite{sugiyama2007covariate}, when comparing representations learned by different models~\citep{huang2007kmm}, and when diagnosing similarities and differences between classes or domains~\citep{morcos2018pwcca}.
While modern practice often compares datasets indirectly (via trained models or embedding distances), such procedures can obscure \emph{why} two datasets appear similar or different.

We take a complementary view: compare datasets through their \emph{geometry}.
Many real-world datasets concentrate near low-dimensional structures, exhibit partially shared latent factors, and contain directions that are specific to one domain or class.
When datasets live in a common ambient feature space (raw pixels, sensor measurements, or fixed embeddings), their geometric relationship becomes meaningful in its own right.

\paragraph{A relation primitive.}
Let $A=[a_1,\dots,a_p]\in\R^{d\times p}$ and $B=[b_1,\dots,b_q]\in\R^{d\times q}$ denote two dataset matrices (columns are observations in $\R^d$).
We characterize similarity through the linear relation
\begin{equation}
Ax = By = z,
\label{eq:intro_cospan}
\end{equation}
where $x\in\R^p$ and $y\in\R^q$ are coefficients and $z\in\R^d$ is a shared ambient vector.
We treat \eqref{eq:intro_cospan} as a primitive object: it encodes compatibility in the ambient space without requiring pointwise correspondences between samples or an invertible mapping between domains.

\paragraph{GSVD yields a joint frame.}
A classical tool for paired subspace analysis is GSVD~\citep{edelman2020gsvd}.
Given $A$ and $B$ with the same row dimension $d$, one common GSVD form provides
\begin{equation}
A = H C U,\qquad
B = H S V,\qquad
C^\top C + S^\top S = I,
\label{eq:intro_gsvd}
\end{equation}
where $H\in\R^{d\times d}$ is invertible (or left-invertible in reduced variants),
$U\in\R^{p\times p}$ and $V\in\R^{q\times q}$ are orthogonal,
and $C$ and $S$ are diagonal or block-diagonal with nonnegative entries.
Intuitively, $H$ defines a shared ambient reference frame, while $(C,S)$ encode how strongly each shared direction contributes to $A$ versus to $B$.

\paragraph{A single alignment angle.}
Our goal is to summarize the relative explanatory power of $A$ versus $B$ for a sample $z$ by a minimal, interpretable quantity.
We define an \emph{alignment angle} $\theta(z)\in[0,\pi/2]$ derived from the GSVD frame: $\theta(z)\approx 0$ means ``more $A$'', $\theta(z)\approx\pi/2$ means ``more $B$'', and $\theta(z)\approx\pi/4$ indicates shared structure.

\paragraph{Contributions.} We summarize our contributions as follows:
\begin{enumerate}
  \item We propose linear relations in the co-span form $Ax=By=z$ as a minimal, geometry-grounded primitive for dataset comparison.
  \item We use GSVD as a natural joint coordinate system for comparing two subspaces, making shared vs.\ dataset-specific directions explicit via $(C,S)$.
  \item We derive an interpretable angle score $\theta(z)$ that quantifies a per-sample diagnostic of relative dataset alignment and supports a binary classification.
  \item We illustrate the geometric behavior of $\theta(z)$ on MNIST through angle distributions and GSVD-derived representative directions.

\end{enumerate}

\paragraph{Organization.}
The remainder of the paper is organized as follows.
Section~\ref{sec:related} reviews related work, and Section~\ref{sec:background} introduces the background on linear relations and GSVD.
Section~\ref{sec:theta} defines the relative alignment angle $\theta(z)$, explains how to compute it efficiently in the GSVD frame, and describes downstream uses such as classification and diagnostics.
Section~\ref{sec:exp} reports experiments on MNIST, including class-conditional angle distributions and representative GSVD directions.
Section~\ref{sec:disc} discusses additional geometric interpretations, followed by conclusions in Section~\ref{sec:concl}.
The appendices provide additional derivations and supporting GSVD properties, including the span-view interpretation and proofs of the main technical results.

\section{Related work} \label{sec:related}
Our work sits at the intersection of (i) classical subspace comparisons and multivariate statistics, (ii) instance alignment objectives, and (iii) modern representation-similarity and transfer-learning perspectives.

\textbf{Subspace geometry and joint decompositions.}
SVD provides canonical coordinates for a single matrix and underpins least-squares and low-rank approximation \citep{golub2013matrix,stewart1993svd}. 
To compare two subspaces, principal angles quantify overlap via the singular values of $Q_A^\top Q_B$ for orthonormal bases $Q_A,Q_B$ \citep{bjorck1973angles,knyazev2002angles}, and Canonical Correlation Analysis (CCA) identifies maximally correlated projections between variable sets \citep{hotelling1936cca}. 
GSVD generalizes SVD to matrix pairs \citep{vanloan1976gsvd,paige1981gsvd}, producing a joint frame $A=HCU$, $B=HSV$ where the diagonal factors $(C,S)$ expose shared versus dataset-specific directions and encode principal-angle-like structure through $(c_i,s_i)$. 
We use this joint frame as an interpretable comparative coordinate system and derive from it a per-sample alignment angle $\theta(z)$. 
Whereas principal angles describe geometry between subspaces, $\theta(z)$ provides a complementary \emph{sample-level diagnostic} indicating whether an observation aligns more strongly with $A$, $B$, or the shared structure.

\textbf{Instance alignment and Procrustes.}
When sample correspondences are available (or assumed), alignment is often formulated as minimizing an instance-wise discrepancy under a restricted transformation family. Procrustes analysis is the canonical example, aligning two point clouds by an orthogonal (or similarity) transform \citep{gower1975procrustes}. In many dataset-comparison settings, however, correspondences are ambiguous or unreliable, motivating our emphasis on relations and subspace geometry. Still, instance alignment is complementary to our approach: the GSVD angles provide a principled way to downweight weakly shared directions before performing alignment (Appendix~\ref{app:span_view}).

\textbf{Representation similarity and dataset distances.}
Recent work compares learned representations using invariances to re-parameterizations. SVCCA and PWCCA adapt CCA to compare neural representations through truncated subspaces and weighting schemes \citep{raghu2017svcca,morcos2018pwcca}, while CKA provides a robust similarity index widely used for layerwise comparisons \citep{kornblith2019cka,cortes2012cka}. In parallel, distributional comparators such as MMD \citep{gretton2012mmd} and optimal transport \citep{peyre2019ot} measure dataset discrepancy directly, and FID is a practical feature-statistics score in generative modeling \citep{heusel2017fid}. These approaches often output a \emph{single} similarity value (or a layerwise matrix) and prioritize invariance or distribution matching. Our goal is complementary: we extract an explicit shared frame (via GSVD), interpretable representative directions, and a  \emph{per-sample} diagnostic $\theta(z)$.

\textbf{Instance-based transfer learning.}
A different but related viewpoint is transfer learning by \emph{reweighting} or \emph{selecting} source instances to better match a target domain. Surveyed in \citet{pan2010transfer}, this line includes boosting-style methods such as TrAdaBoost \citep{zheng2020improved} and covariate-shift / importance-weighting approaches \citep{sugiyama2007covariate,huang2007kmm}. Our angle score $\theta(z)$ plays a conceptually similar role as a \emph{relative compatibility} signal: it can be used to flag samples that are more consistent with one dataset than the other, which is useful for auditing, filtering, or prioritizing instances when comparing or transferring between domains.

\section{Background: Linear relations and GSVD} \label{sec:background}

This section presents two equivalent views of linear relations: the \emph{co-span} formulation, which underlies the rest of the paper, and the \emph{span} formulation, developed further in Appendix~\ref{app:span_view}. We also introduce GSVD as a joint coordinate framework for expressing relations between datasets, highlighting the block structure of $(C,S)$ that separates shared from dataset-specific directions.

\subsection{Linear relations: co-span vs.\ span}
A linear relation between the two subspaces can be formalized as
\begin{equation}
\mathcal{R}_{\text{co-span}} \;=\; \{(x,y)\in\R^p\times\R^q : Ax = By\}.
\label{eq:cospan_relation}
\end{equation}
Here, $x$ and $y$ denote coordinate vectors, and the constraint imposes that both sides yield the \emph{same} ambient vector $z := Ax = By \in \R^d$. From a geometric perspective, this encodes the intersection (and approximate intersection) structure between $\col(A)$ and $\col(B)$, which denote the column spaces of $A$ and $B$, respectively; it can be naturally interpreted as an \emph{alignment of dimensions} by enforcing representability in a shared ambient space.  Appendix~\ref{app:notation} summarizes the notation used throughout the paper.


Equivalently, one may introduce an explicit shared latent parameter $w$:
\begin{equation}
\mathcal{R}_{\text{span}} \;=\; \{(x,y) : \exists\, w \text{ s.t. } x = F w,\;\; y = G w\},
\label{eq:span_relation}
\end{equation}
for suitable linear maps $F,G$. This view is closer to \emph{instance alignment}, but in many dataset-comparison settings correspondences are unreliable, motivating the relation-first view \eqref{eq:cospan_relation}.

\subsection{GSVD as a joint geometry}

To compare two subspaces simultaneously, GSVD produces a \emph{joint coordinate system}. In one common form, given matrices $A$ and $B$, GSVD produces the decomposition given by~\eqref{eq:intro_gsvd}. 
Intuitively, $H$ defines a shared ambient reference frame, while the diagonal factors $C^{d\times p}$ and $S^{d\times q}$ quantify how strongly each shared direction contributes to $A$ and $B$. In particular, the diagonal (or block-diagonal) factors $C$ and $S$ are ordered so that entries of $C$ decrease and entries of $S$ increase, making ``more-$A$'' versus ``more-$B$'' directions explicit.

\textbf{Block structure of $C$ and $S$.}\label{def:gsvd_block}
In the GSVD form \eqref{eq:intro_gsvd}, the diagonal (or block-diagonal) factors $C$ and $S$
can be written by isolating a central diagonal block (the shared subspace) and padding
with identity/zero blocks:
\begin{equation}
C \;=\;
\begin{bmatrix}
I_{r} & 0 & 0 \\
0 & \widetilde C & 0 \\
0 & 0 & 0_{t}
\end{bmatrix},
\qquad
S \;=\;
\begin{bmatrix}
0_{r} & 0 & 0 \\
0 & \widetilde S & 0 \\
0 & 0 & I_{t}
\end{bmatrix},
\label{eq:cs_blocks}
\end{equation}
where $I_r$ and $I_t$ are identity matrices, $0_r$ and $0_t$ are zero matrices,
and $\widetilde C,\widetilde S \in \mathbb{R}^{k\times k}$ are diagonal with strictly positive entries.
The sizes satisfy $r+k+t=d$.

The $C$ and $S$ elements are ordered as follows:
\begin{align*}
 c_{11} \geq  c_{22} \geq \dots \geq  c_{dd},\\
 s_{11} \leq  s_{22} \leq \dots \leq  s_{dd}.
\end{align*}

\textbf{Shared vs.\ specific directions.}
Directions where a diagonal entry of $C$ dominates correspond to geometry primarily explained by $A$; directions where $S$ dominates are primarily explained by $B$; comparable magnitudes indicate shared structure.

\section{A relative alignment angle $\theta(z)$ and its uses}
\label{sec:theta}

This section defines the alignment angle associated with a fixed instance, presents an algorithm for its computation, characterizes the instances that attain its maximum and minimum values, and presents two downstream applications.

\subsection{Defining the Alignment Angle}
We now turn the relation between two datasets $A$ and $B$ into a single, interpretable score.
For a given $z \in \operatorname{col}(A)\cap \operatorname{col}(B)$, i.e., for which the
co-span relation $Ax = By = z$ is feasible, define the fiber
\begin{equation}
\mathcal{R}_{\text{co-span}}(z)
\;=\;
\{(x,y)\in\R^p\times\R^q : Ax = By = z\}.
\label{eq:cospan_fiber}
\end{equation}

If $z$ does not lie in this intersection, the relation is undefined and the
sample is considered \emph{non-relational} with respect to $(A,B)$.

\begin{definition}[Alignment angle] \label{eq:keydefinition} The alignment angle  of $z$ given   matrices $A$ and $B$   is defined as
\begin{equation}
\theta(z)
\;:=\;
\arctan\!\left(\frac{\|x\|_2}{\|y\|_2}\right) \in\;
\bigl[0,\tfrac{\pi}{2}\bigr],  
(x,y)\in\mathcal{R}_{\text{co-span}}(z), x\perp \Ker(A), y \perp \Ker(B).
\label{eq:theta_def}
\end{equation}
\end{definition}
\paragraph{Interpretation.}
$\theta(z)\approx 0$ means $z$ is explained by $A$ with smaller coefficient norm than by $B$ (``more $A$'');
$\theta(z)\approx\pi/2$ means the symmetric situation (``more $B$'');
$\theta(z)\approx\pi/4$ indicates comparable explanatory strength (shared structure).

The constraints $x\perp\Ker(A)$ and $y\perp\Ker(B)$ pick the minimum-$\ell_2$ norm coefficients among all pairs satisfying $Ax=By=z$.
Hence $\|x\|_2$ and $\|y\|_2$ are canonical ``costs'' to represent the same $z$ using $A$ versus $B$, and the ratio $\|x\|_2/\|y\|_2$ compares these costs.
Using $\arctan$ maps this positive ratio to a bounded, symmetric score in $[0,\pi/2]$, with equal costs $\|x\|_2=\|y\|_2$ giving $\theta=\pi/4$.

\subsection{Computing $\theta(z)$ in the GSVD frame} \label{sec:gsvdmain}

Next, we express the alignment score~\eqref{eq:keydefinition} in terms of the GSVD factors, obtaining a closed-form expression that is easy to compute and geometrically interpretable: $H$ defines a shared ambient basis, while $(C,S)$ weight the relative contributions of $A$ and $B$ along each direction.

Let $(H,C,S,U,V)$ denote the GSVD factors in \eqref{eq:intro_gsvd}.
For an arbitrary sample $z\in\R^d$, compute its coordinates in the shared frame by
\begin{equation}
c(z)\;:=\;\argmin_{c\in\R^d}\;\norm{Hc-z}_2^2
\qquad\Rightarrow\qquad
c(z)=H^\dagger z.
\label{eq:c_def}
\end{equation}
Define the GSVD-weighted costs
\begin{equation}
a(z):=\norm{C^\dagger c(z)}_2,\qquad
b(z):=\norm{S^\dagger c(z)}_2,
\label{eq:ab_def}
\end{equation}
where $\dagger$ denotes the Moore--Penrose pseudoinverse (optionally truncated for numerical stability).
Because $U$ and $V$ are orthogonal, these costs correspond to the norms of canonical coefficient choices in $\mathcal{R}_{\text{co-span}}(z)$ (up to the same orthogonal transforms).

\begin{theorem}[GSVD produces alignment angles]   
\begin{equation}
\theta(z)=\arctan\!\left(\frac{a(z)}{b(z)}\right)
=
\arctan\!\left(
\frac{\norm{C^\dagger\,c(z)}_2}{\norm{S^\dagger\,c(z)}_2}
\right),
\end{equation}
where  $c(z)=H^\dagger z$  
and  $(H,C,S,U,V)$ are  the GSVD factors in~\eqref{eq:intro_gsvd}.
\end{theorem}
\begin{proof} See Appendix~\ref{sec:appanalytproof}.
\end{proof}

\subsection{Finding Extreme Directions}
\label{sec:optimization}

So far, the pipeline is \emph{forward}: given a sample $z$, we compute $\theta(z)$ as a relative-alignment diagnostic.
We can also ask the inverse question: \emph{which vectors $z$ are maximally $A$-like or maximally $B$-like under the GSVD geometry?}
Equivalently, we seek
\begin{equation}
z_{\max}\in\argmax_{z\in\mathcal{Z}}\ \theta(z),
\qquad
z_{\min}\in\argmin_{z\in\mathcal{Z}}\ \theta(z), \label{eq:optmproblem}
\end{equation}
over a constraint set $\mathcal{Z}$ (e.g., $\|z\|_2=1$, or $z$ restricted to the span of training data).
This viewpoint produces representative ``extreme'' directions for visualization and diagnostics, complementing the per-sample scoring setting of Section~\ref{sec:gsvdmain}.

In what follows, we characterize $z_{\textrm{max}}$ and $z_{\textrm{min}}$  in terms of the GSVD structure. In particular, $z_\text{max}$ and $z_\text{min}$ correspond to specific columns of the shared matrix $H$: namely, $h_{r+k}$ and $h_{r+1}$ respectively. 
These vectors represent the last and first shared generalized components of the datasets, where “shared” in this instance refers to components associated with indices for which both diagonal matrices $C$ and $S$ have nonzero entries. As such the GSVD provides not only the coordinate system used to calculate \eqref{eq:theta_def}, but that same basis encodes the solution for  optimization problem~\eqref{eq:optmproblem}.

Let $\tilde{x}=(\tilde x_1, 
\tilde x_2, 
\tilde x_3)$ and $\tilde{y}=(\tilde y_1, 
\tilde y_2, 
\tilde y_3)$, where the dimensions of the components of $\tilde{x}$ and  $\tilde{y}$  correspond to the blocks of $C$ and $S$ given by~\eqref{eq:cs_blocks}, i.e.,  $\tilde{x}_1 \in \mathbb{R}^{r}$, $\tilde{x}_2 \in \mathbb{R}^{k}$ and $\tilde{x}_3 \in \mathbb{R}^{t}$. Similarly,  $\tilde{y}_1 \in \mathbb{R}^{r}$, $\tilde{y}_2 \in \mathbb{R}^{k}$ and $\tilde{y}_3 \in \mathbb{R}^{t}$. 
\begin{theorem}[GSVD produces extreme directions] \label{thm:ratio-max}
Given datasets $A$ and $B$, the extreme directions $z_{\textrm{max}}$ and $z_{\textrm{min}}$
that maximize and minimize the alignment angle $\theta(z)$, respectively, are unique and given by
$z_{\textrm{max}}=Ax_{\textrm{max}}=By_{\textrm{max}}$ and
$z_{\textrm{min}}=Ax_{\textrm{min}}=By_{\textrm{min}}$, where
$z_{\textrm{max}} = h_{r+k}$, $z_{\textrm{min}} = h_{r+1}$ and
\[
x_{\textrm{max}} =
U^T\begin{bmatrix}
0\\
\widetilde S\widetilde C^{-1}e_{k}\\
0
\end{bmatrix},
\quad
y_{\textrm{max}} =
V^T\begin{bmatrix}
0\\
e_{k}\\
0
\end{bmatrix},
\quad
x_{\textrm{min}} =
U^T\begin{bmatrix}
0\\
\widetilde C\widetilde S^{-1}e_{1}\\
0
\end{bmatrix},
\quad
y_{\textrm{min}} =
V^T\begin{bmatrix}
0\\
 e_{1}\\
0
\end{bmatrix}.
\]
$(H,C,S,U,V)$ are the GSVD factors in~\eqref{eq:intro_gsvd}, $e_i$ is the $i$-th canonical basis vector,
i.e., $(e_i)_j = 1$ if $j=i$ and $(e_i)_j = 0$ otherwise, and $h_i$ is the $i$-th column of matrix $H$.
\end{theorem}

\begin{proof}
See Appendix~\ref{sec:appanalytproof}.
\end{proof}


\begin{remark}[Deflation / subsequent directions] \label{def:deflation}
A standard way to obtain subsequent maximizers is to augment the optimization problem underlying
Theorem~\ref{thm:ratio-max} with the additional constraint that $x$ be orthogonal to all previously
computed maximizers.
Let $X_{\text{prev}}$ denote the matrix whose columns are the previously obtained maximizers, and
impose the constraint $X_{\text{prev}}^{\top}x=0$ in the maximization problem.
This procedure is commonly referred to as \emph{deflation}, in analogy with extracting successive
principal components in PCA.
An explicit orthogonality constraint on $y$ is unnecessary, since its orthogonality is induced by
that of $x$.
For further details, we refer the reader to~\citep{chu1997variational}.
\end{remark}

\subsection{Downstream uses: diagnostics and illustrative classification}
\label{sec:uses}
\textbf{Binary classification (relative geometry).}\label{uses_binary}
Given two domains (e.g., digit ``1'' vs.\ digit ``5''), we build matrices $A$ and $B$ from their training samples (or low-rank bases thereof). For a test sample $z$, compute $\theta(z)$ and predict
\[
\hat{\ell}(z)=
\begin{cases}
A, & \theta(z)\le \tau,\\
B, & \theta(z)>\tau,
\end{cases}
\]
with $\tau=\pi/4$ as a symmetry default, or tuned on validation data.  The pseudo-code of the classifier is presented in Appendix~\ref{app:algo} (see Algorithm \ref{alg:theta}).

We emphasize that this rule is presented only as an illustrative consequence of the alignment score. The goal of the method is not to propose a competitive classifier, but to provide a geometrically interpretable diagnostic of how strongly a sample aligns with each dataset.

\textbf{Diagnostics.}\label{uses_diagnostics}
Samples whose angles conflict with their nominal labels (e.g., a class-$A$ sample with $\theta(z)$ close to $\pi/2$) are candidates for auditing. The GSVD frame also provides interpretable directions, which can be used to visualize shared versus dataset-specific structure.

\textbf{Decomposition Complexity.}\label{uses_complexity}
GSVD requires $O(d^3)$ time, which presents a bottleneck for large-scale databases. Since the decomposition is performed as a preprocessing step, however, it does not impact cost at inference time. Broader adoption and use cases of the GSVD in machine learning contexts may motivate the development of more efficient algorithms, leaving reduced decomposition complexity as a promising direction for future work. 

\textbf{Numerical considerations.}\label{uses_numerical_considerations}
In practical implementations the diagonal entries of $C$ and $S$ may contain very small values, making the pseudoinverses $C^\dagger$ and $S^\dagger$ sensitive to noise. In our experiments we use a truncated pseudoinverse that ignores singular values below a small tolerance. A systematic analysis of stability with respect to rank truncation, preprocessing, and noise remains an interesting direction for future work.

\section{Experiments on MNIST}
\label{sec:exp}

We describe a simple protocol on MNIST \citep{mnist} to illustrate how $\theta$ behaves as an angle score for classification. Additional results on MNIST-Fashion are reported in Appendix~\ref{app:fashion_mnist}. The purpose of this experiment is not to benchmark classification performance but to illustrate the geometric behavior of the alignment score. MNIST provides a controlled setting in which the induced subspace structure can be visualized and interpreted. The same pipeline can be applied to other representations, including learned feature embeddings, which we leave for future work.

\textbf{Setup.}
We choose two digits $a$ and $b$ (e.g., ``1'' vs.\ ``5'').
Each image is vectorized in $\mathbb{R}^{784}$ (from $28\times 28$ pixels) and mean-centered.
MNIST consists of 60{,}000 training images and 10{,}000 test images.
We construct the matrix $A$ by stacking $p=900$ training images from digit $a$ as columns,
and analogously construct $B$ by stacking $q=800$ training images from digit $b$,
with the columns selected uniformly at random from the training set. 

\begin{figure}[t]
\centering
\setlength{\tabcolsep}{2pt}
\begin{tabular}{cccc}
\safeincludegraphics[height=2.4cm]{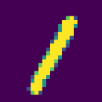} &
\safeincludegraphics[height=2.4cm]{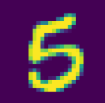} &
\safeincludegraphics[height=2.4cm]{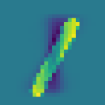} &
\safeincludegraphics[height=2.4cm]{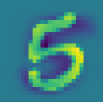} \\
(a) & (b) & (c) & (d)
\end{tabular}
\caption{Representative MNIST digit samples for the GSVD pipeline. (a,b) raw digits ``1'' and ``5''; (c,d) corresponding vectors after mean-centering.}
\label{fig:mnist_examples}
\end{figure}

\textbf{Evaluation.}
The histograms reported in Figure~\ref{fig:theta_distributions} are obtained by computing $\theta(z)$ for all test samples $z$ belonging to the two selected digits,
i.e., for all available instances of digits $a$ and $b$ in the MNIST test set, using ~\ref{eq:theta_gsvd}. We report the empirical class-conditional distributions of $\theta(z)$ as a diagnostic of relative geometric separability.

\textbf{Interpreting the $\theta(z)$ histograms. }
If the two digits induce clearly different subspaces, the histograms concentrate on different angle ranges (small angles for the first digit in the pair and large angles for the second). Overlap between the histograms indicates directions that are similarly explained by both subspaces, and thus a more ambiguous relative geometry.

\begin{figure}[t]
\centering
\begin{subfigure}[t]{0.48\textwidth}
  \centering
  \safeincludegraphics[width=\linewidth]{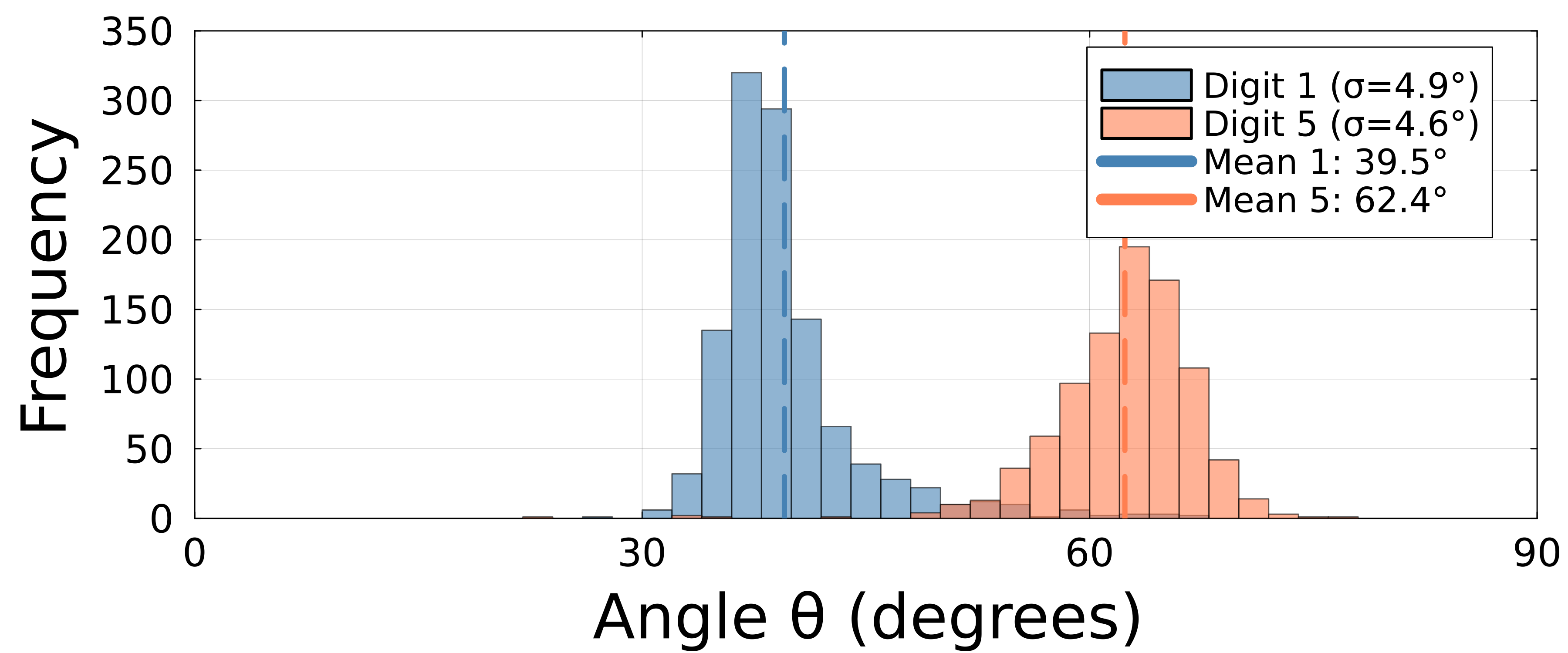}
  \caption{Digits ``1'' vs ``5''}
\end{subfigure}\hfill
\begin{subfigure}[t]{0.48\textwidth}
  \centering
  \safeincludegraphics[width=\linewidth]{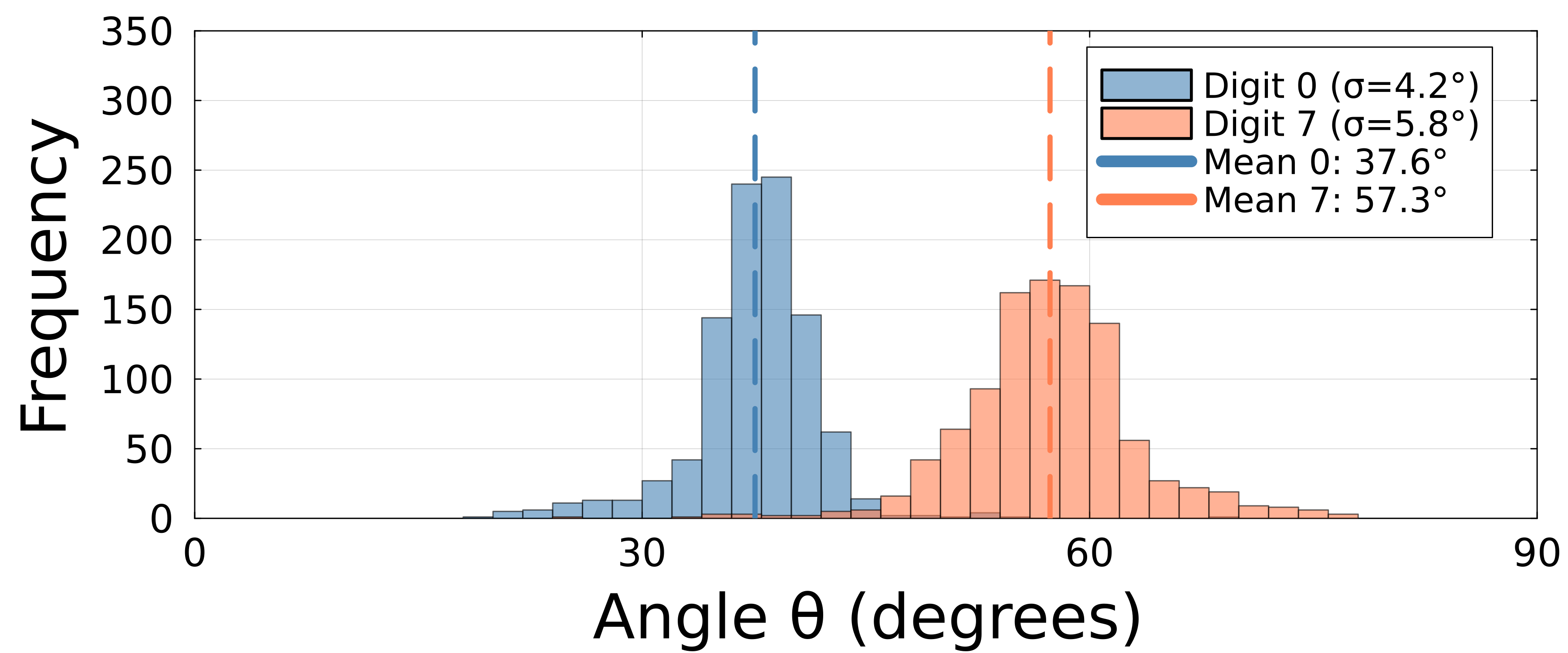}
  \caption{Digits ``0'' vs ``7''}
\end{subfigure}

\vspace{0.6em}

\begin{subfigure}[t]{0.48\textwidth}
  \centering
  \safeincludegraphics[width=\linewidth]{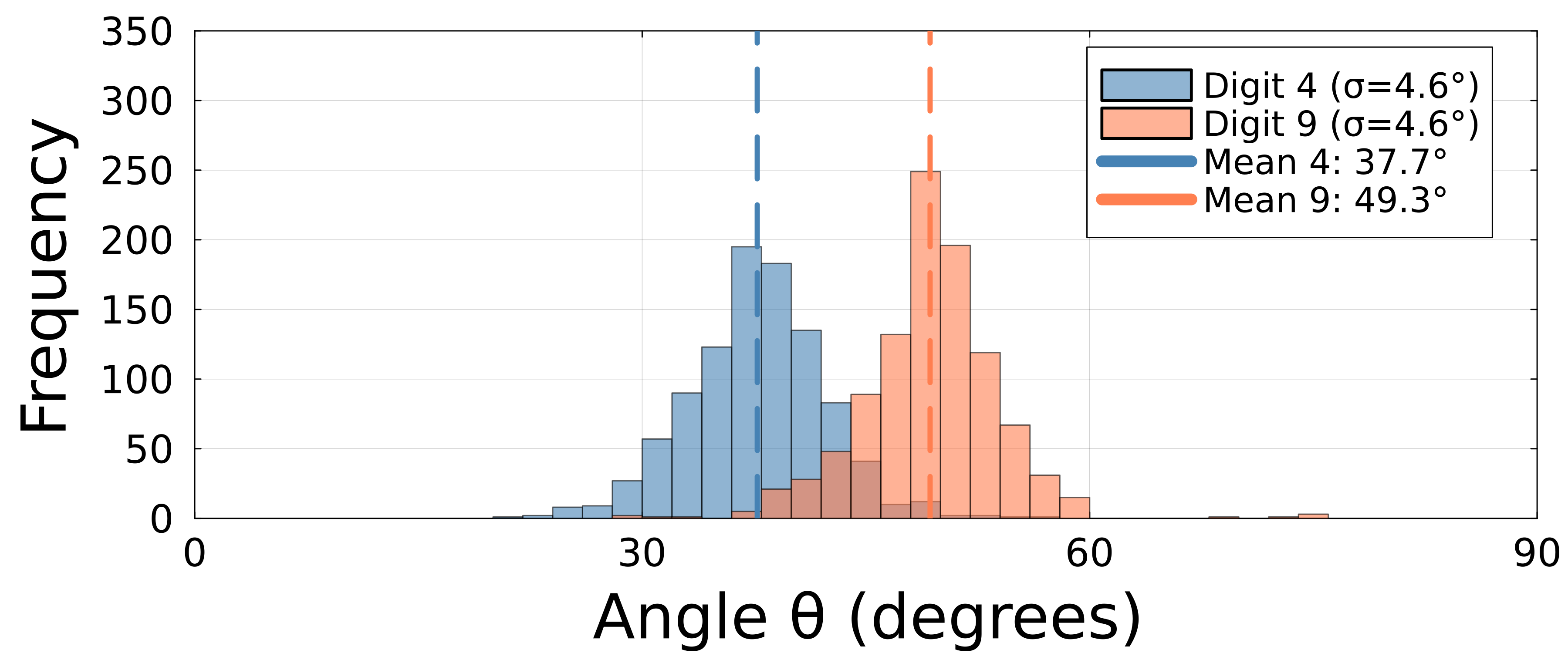}
  \caption{Digits ``4'' vs ``9''}
\end{subfigure}\hfill
\begin{subfigure}[t]{0.48\textwidth}
  \centering
  \safeincludegraphics[width=\linewidth]{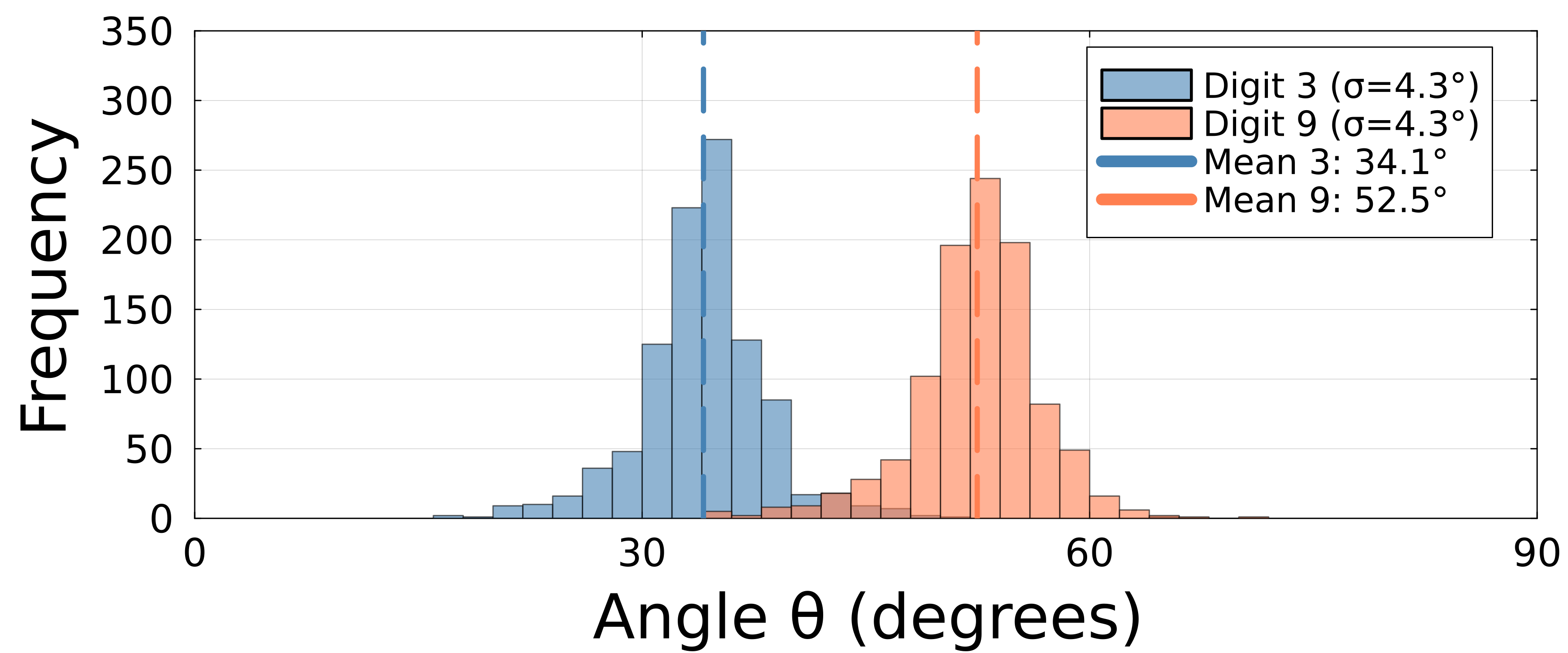}
  \caption{Digits ``3'' vs ``9''}
\end{subfigure}

\caption{Empirical distributions of the angle $\theta(z)$ on the MNIST test set for four digit pairs. Values closer to $0$ indicate stronger alignment with the first digit in the pair, while values closer to $90^\circ$ indicate stronger alignment with the second.}
\label{fig:theta_distributions}
\end{figure}

\begin{figure}[t] \centering {\setlength{\tabcolsep}{0pt} \begin{tabular}{ccc} \includegraphics[height=2.5cm]{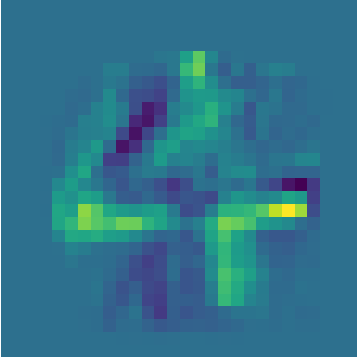} & \includegraphics[height=2.5cm]{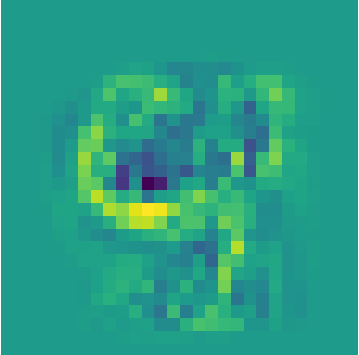} & \includegraphics[height=2.5cm]{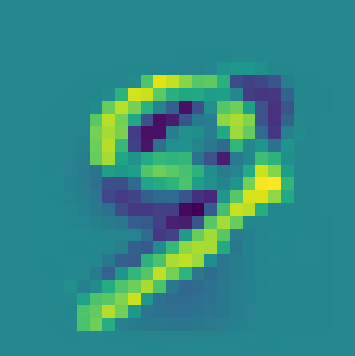} \\ (a) & (b) & (c) \end{tabular} } \caption{Representative $H$ component directions reconstructed in the image space as 28x28 images, using the viridis colormap,  obtained by the GSVD-based   optimization in Section~\ref{sec:optimization}:   (a)   representation of a solution from the minimization problem indicating a more ``4-like" direction;   (c)   representation of the solution from the maximization problem indicating a more ``9-like" direction;  (b) reconstruction of a shared direction that encapsulates the structure of both ``4" and ``9".}\label{fig:mnist_extremes} \end{figure}

For a scalar, threshold-free summary of the separation between the class-conditional
$\theta$ distributions, we report Fisher--Rao distances between the corresponding histograms
in Appendix~\ref{sec:fisher_rao} (Table~\ref{tab:fr_mnist_pairs}).

\textbf{Interpreting the $\theta(z)$ extremes.}
The extreme columns correspond to directions of the shared GSVD frame $H$ that maximize or minimize the alignment angle $\theta(z)$ as derived in Section~\ref{sec:optimization}. The columns of $H$ can represent directions ordered by increasing relative-alignment angle $\theta(z)$, ranging from directions that are maximally aligned with $A$ to those that are maximally aligned with $B$. Intermediate columns capture directions that jointly reflect structural characteristics of both subspaces.

Figure~\ref{fig:mnist_extremes} illustrates this behavior for matrices $A$ and $B$ constructed from MNIST images of digits 4 and 9, respectively. The $A$-aligned extreme exhibits a prototypical ``4-like'' structure with sharp, edge patterns, while the $B$-aligned extreme emphasizes rounded contours and a diagonal component. The central image reveals a
blended representation, combining elements shared by both digits, and thus illustrates directions
along which samples from the two classes are similarly explained.

These visual structures are inherently data-dependent: the appearance of extreme and intermediate directions is determined by the specific instances used to construct $A$ and $B$. Consequently, they should be interpreted as empirical manifestations of the underlying class geometry rather than canonical templates.

\section{Discussion: geometry, scale, and information geometry} \label{sec:disc}

This section connects the alignment angle $\theta$ to probabilistic and information-geometric interpretations. While $\theta(z)$ was originally motivated by linear-algebraic considerations, it naturally induces probabilistic structures whose geometry can be studied using tools from information geometry.

\textbf{From angles to probabilistic geometry.}
Our goal is to connect the alignment angle $\theta$, originally motivated by linear-algebraic arguments, to probabilistic and information-geometric interpretations.
We do so in two complementary steps.

First, we show that each sample-wise angle $\theta(z)$ induces a Bernoulli posterior whose
Fisher--Rao geometry is linear in $|\theta(z)-\theta(z')|$.
Second, we move from individual samples to distributions over angles and study the
Fisher--Rao distance between $\theta$-histograms, showing that this distance is driven primarily
by mass near $\theta=\pi/4$, while samples near $0$ or $\pi/2$ correspond to confident regions
with negligible contribution.   

The results in this section are presented using Fisher--Rao geometry.
Alternatively, they can also be derived via the Hellinger geometry,
as discussed in Appendix~\ref{app:hellinger_fr}.

\subsection{Per-sample Fisher--Rao geometry induced by $\theta$}
$\theta(z)=\arctan(a(z)/b(z))$.
We interpret inverse-squared costs as evidences and adopt the parametric posterior model
\begin{equation}
P(A\mid \theta)=\frac{1/a(z)^2}{1/a(z)^2+1/b(z)^2}
=\frac{b(z)^2}{a(z)^2+b(z)^2}
=\cos^2\theta,
\qquad
P(B\mid \theta)=\sin^2\theta,
\label{eq:posterior_cos2}
\end{equation}
so that $P(A\mid\theta)/P(B\mid\theta)=\cot^2\theta$.
This posterior coincides with a Bayes posterior based on likelihoods
$p(\theta\mid A)$ and $p(\theta\mid B)$ whenever
\begin{equation}
\frac{p(\theta\mid A)}{p(\theta\mid B)}=\frac{\pi_B}{\pi_A}\cot^2\theta,
\label{eq:lr_cot2}
\end{equation}
where $\pi_A=P(A)$ and $\pi_B=P(B)$.

Under the square-root (Bhattacharyya) embedding
$\psi(\theta)=(\cos\theta,\sin\theta)$ of the Bernoulli simplex,
the Fisher--Rao distance between the posteriors induced by two samples $z$ and $z'$ is
\begin{equation}
d_{\mathrm{FR}}(p(z),p(z'))
=2\acos\!\big(\langle \psi(\theta(z)),\psi(\theta(z'))\rangle\big)
=2\acos\!\big(\cos(\theta(z)-\theta(z'))\big)
=2|\theta(z)-\theta(z')|,
\label{eq:fr_sample}
\end{equation}
since $\theta\in[0,\pi/2]$.
Thus, at the sample level, differences in $\theta$ correspond exactly to Fisher--Rao distances
between induced Bernoulli posteriors.

\subsection{Fisher--Rao distance between $\theta$-histograms and posterior ambiguity} \label{sec:fisherraodist}
We now move from individual samples to distributions over angles.
Fix a binning of $[0,\pi/2]$ into $m$ bins and let
$P=(P_i)$ and $Q=(Q_i)$ be normalized histograms estimating the class-conditional
likelihoods $P(\theta\in \text{bin }i\mid A)$ and $P(\theta\in \text{bin }i\mid B)$.
Let $\pi_A$ and $\pi_B$ denote the class priors and define the mixture weights
$m_i=\pi_A P_i+\pi_B Q_i$.
For bin index $I=i$, define the binary label random variable
$Y=\mathbf{1}\{\text{class}=A\}$, which equals $1$ if the sample belongs to class $A$
and $0$ otherwise.
The posterior distribution of $Y$ given $I=i$ is
\begin{equation}
(Y\mid I=i)\sim \mathrm{Bernoulli}(r_i),
\qquad
r_i=P(A\mid I=i)=\frac{\pi_A P_i}{\pi_A P_i+\pi_B Q_i}. \label{eq:posty}
\end{equation}
The Fisher--Rao distance  between the histograms, denoted by $d_{\mathrm{FR}}(P,Q)$, and defined as $d_{\mathrm{FR}}(P,Q)
=2\acos\!\left(\sum_{i=1}^m \sqrt{P_iQ_i}\right)$,   satisfies
\begin{equation}
d_{\mathrm{FR}}(P,Q)
=2\acos\!\left(\sum_{i=1}^m m_i \sqrt{\frac{r_i(1-r_i)}{\pi_A\pi_B}}\right)
=2\acos\!\left(
\frac{\mathbb{E}_{I\sim m}\!\left[\mathrm{Std}(Y{\mid} I)\right]}
{\sqrt{\pi_A\pi_B}}
\right).  
\end{equation}
Hence, $d_{\mathrm{FR}}(P,Q)$ is a monotone \emph{decreasing} function of the
\emph{expected posterior standard deviation}
$\mathbb{E}[\mathrm{Std}(Y\mid I)]$:
bins with $r_i\approx\tfrac12$ contribute most to the standard deviation, while bins with
$r_i\approx0$ or $1$ contribute little.
Under the parametric posterior~\eqref{eq:posterior_cos2} and equal priors,
samples with $\theta\approx\pi/4$ concentrate probability mass in posterior bins with
$r_i\approx\tfrac12$,
whereas samples with $\theta\approx 0$ or $\pi/2$ concentrate probability mass in bins with
$r_i$ close to $1$ or $0$, respectively.
This establishes a direct link between the linear-algebraic alignment angle $\theta$,
the distance between class-conditional likelihoods, and posterior uncertainty. 

\noindent\textbf{Numerical illustration. }
Figure~\ref{fig:theta_distributions} provides an empirical illustration of the theoretical picture
developed above. For digit pairs with clearer geometric separation (e.g., ``1'' vs.\ ``5'',   ``3'' vs.\ ``9'' and
``0'' vs.\ ``7''), the class-conditional $\theta$-histograms concentrate toward opposite extremes,
with relatively little mass near $\theta=\pi/4$, indicating confident posterior regions.
In contrast, visually similar pairs (e.g., ``4'' vs.\ ``9'') exhibit increased mass around $\theta\approx\pi/4$, reflecting higher posterior
ambiguity. This behavior is consistent with the Fisher--Rao analysis: overlap near $\pi/4$
corresponds to bins with $r_i\approx\tfrac12$ that dominate the Bhattacharyya coefficient and
reduce histogram-level Fisher--Rao distance.

Together, the results  presented in this section complete the connection between   
$\theta$ as introduced in Definition~\ref{eq:keydefinition} and probabilistic notions of posterior uncertainty through Fisher--Rao geometry.

\section{Conclusion} \label{sec:concl}

We introduced a geometry-grounded approach to dataset comparison built around a single primitive,
the co-span relation $Ax=By=z$, and operationalized it through a GSVD joint frame. The resulting
alignment angle $\theta(z)$ provides a minimal, interpretable summary of \emph{relative} explanatory
power: small angles indicate that $z$ is represented more economically by $A$, large angles favor $B$,
and values near $\pi/4$ correspond to shared structure. Beyond yielding a simple illustrative decision rule, the
same GSVD factors expose representative directions (including extremes) that act as visual
diagnostics of what is shared and what is dataset-specific. 
Empirically, MNIST digit pairs exhibit characteristic class-conditional $\theta$ distributions whose
overlap (or lack thereof) matches intuitive notions of separability. The Fisher--Rao distance between
$\theta$-histograms complements these plots with a scalar, information-geometric separation metric,
linking “dataset similarity” to posterior ambiguity induced by the angle.

This work is limited to two domains and controlled pixel-level datasets. Several extensions are left for future work. First, the framework can be generalized from dataset pairs to multiple domains, for instance by aggregating pairwise angles into a simplex-valued score or using multiway GSVD-style constructions. Second, the robustness of the alignment angle has not yet been characterized; future work should study the sensitivity of $\theta$ to preprocessing, rank selection, and the regularization or truncation of $C^\dagger$ and $S^\dagger$, as well as its behavior under noise and partial mismatch when $z \notin \mathrm{col}(A)\cap\mathrm{col}(B)$. Finally, while our experiments focus on raw pixel representations for interpretability, the framework applies more broadly to arbitrary feature spaces. In particular, matrices $A$ and $B$ can be constructed from feature embeddings produced by pretrained models such as convolutional networks or transformers. In such representations the linear subspace assumption may be more appropriate, since modern embeddings are designed to linearize semantic structure. Exploring the behavior of the alignment angle in such representation spaces is therefore a natural direction for future work.
\bibliographystyle{iclr2026_conference}
\bibliography{references}

@book{golub2013matrix,
  title={Matrix Computations},
  author={Golub, Gene and Van Loan, Charles},
  edition={4th},
  year={2013},
  publisher={Johns Hopkins University Press}
}

@article{vanloan1976gsvd,
  title={Generalizing the singular value decomposition},
  author={Van Loan, Charles},
  journal={SIAM Journal on Numerical Analysis},
  volume={13},
  number={1},
  pages={76--83},
  year={1976}
}

@article{edelman2020gsvd,
  title={{The GSVD: Where are the ellipses?, matrix trigonometry, and more}},
  author={Edelman, Alan and Wang, Yuyang},
  journal={SIAM Journal on Matrix Analysis and Applications},
  volume={41},
  number={4},
  pages={1826--1856},
  year={2020},
  publisher={SIAM}
}

@article{paige1981gsvd,
  title={Towards a generalized singular value decomposition},
  author={Paige, Christopher  and Saunders, Michael},
  journal={SIAM Journal on Numerical Analysis},
  volume={18},
  number={3},
  pages={398--405},
  year={1981}
}

@article{stewart1993svd,
  title={On the early history of the singular value decomposition},
  author={Stewart, G. W.},
  journal={SIAM Review},
  volume={35},
  number={4},
  pages={551--566},
  year={1993}
}

@article{gower1975procrustes,
  title={{Generalized Procrustes analysis}},
  author={Gower, J. C.},
  journal={Psychometrika},
  volume={40},
  number={1},
  pages={33--51},
  year={1975}
}

@article{hotelling1936cca,
  title={Relations between two sets of variates},
  author={Hotelling, Harold},
  journal={Biometrika},
  volume={28},
  number={3/4},
  pages={321--377},
  year={1936}
}

@inproceedings{mnist,
  title={Gradient-based learning applied to document recognition},
  author={LeCun, Yann and Bottou, L{\'e}on and Bengio, Yoshua and Haffner, Patrick},
  booktitle={Proceedings of the IEEE},
  volume={86},
  number={11},
  pages={2278--2324},
  year={1998}
}

@article{bjorck1973angles,
  title={Numerical methods for computing the angles between linear subspaces},
  author={Bj{\"o}rck, {\AA}ke and Golub, Gene H.},
  journal={Mathematics of Computation},
  volume={27},
  number={123},
  pages={579--594},
  year={1973}
}

@article{knyazev2002angles,
  title={Principal angles between subspaces in an $A$-based scalar product: algorithms and perturbation estimates},
  author={Knyazev, Andrew V. and Argentati, Merico E.},
  journal={SIAM Journal on Scientific Computing},
  volume={23},
  number={6},
  pages={2008--2040},
  year={2002}
}

@article{chu1997variational,
  title={On a Variational Formulation of the Generalized Singular Value Decomposition},
  author={Chu, Moody T. and Funderlic, Robert E. and Golub, Gene H.},
  journal={SIAM Journal on Matrix Analysis and Applications},
  volume={18},
  number={4},
  pages={1082--1092},
  year={1997},
  publisher={SIAM}
}

@article{raghu2017svcca,
  title={SVCCA: Singular Vector Canonical Correlation Analysis for Deep Learning Dynamics and Interpretability},
  author={Raghu, Maithra and Gilmer, Justin and Yosinski, Jason and Sohl-Dickstein, Jascha},
  journal={arXiv preprint arXiv:1706.05806},
  year={2017}
}

@inproceedings{morcos2018pwcca,
  title={Insights on representational similarity in neural networks with canonical correlation},
  author={Morcos, Ari S. and Raghu, Maithra and Bengio, Samy},
  booktitle={Advances in Neural Information Processing Systems},
  volume={31},
  year={2018}
}

@inproceedings{kornblith2019cka,
  title={Similarity of Neural Network Representations Revisited},
  author={Kornblith, Simon and Norouzi, Mohammad and Lee, Honglak and Hinton, Geoffrey},
  booktitle={International Conference on Machine Learning},
  year={2019}
}

@article{cortes2012cka,
  title={Algorithms for learning kernels based on centered alignment},
  author={Cortes, Corinna and Mohri, Mehryar and Rostamizadeh, Afshin},
  journal={Journal of Machine Learning Research},
  volume={13},
  number={Mar},
  pages={795--828},
  year={2012}
}

@article{gretton2012mmd,
  title={A Kernel Two-Sample Test},
  author={Gretton, Arthur and Borgwardt, Karsten M. and Rasch, Malte J. and Sch{\"o}lkopf, Bernhard and Smola, Alexander},
  journal={Journal of Machine Learning Research},
  volume={13},
  number={Mar},
  pages={723--773},
  year={2012}
}

@article{peyre2019ot,
  title={Computational Optimal Transport},
  author={Peyr{\'e}, Gabriel and Cuturi, Marco},
  journal={Foundations and Trends in Machine Learning},
  volume={11},
  number={5--6},
  pages={355--607},
  year={2019}
}

@inproceedings{heusel2017fid,
  title={{GANs Trained by a Two Time-Scale Update Rule Converge to a Local Nash Equilibrium}},
  author={Heusel, Martin and Ramsauer, Hubert and Unterthiner, Thomas and Nessler, Bernhard and Hochreiter, Sepp},
  booktitle={Advances in Neural Information Processing Systems},
  year={2017}
}

@article{pan2010transfer,
  title={A Survey on Transfer Learning},
  author={Pan, Sinno Jialin and Yang, Qiang},
  journal={IEEE Transactions on Knowledge and Data Engineering},
  volume={22},
  number={10},
  pages={1345--1359},
  year={2010}
}

@article{zheng2020improved,
  title={{Improved TrAdaBoost and its application to transaction fraud detection}},
  author={Zheng, Lutao and Liu, Guanjun and Yan, Chungang and Jiang, Changjun and Zhou, Mengchu and Li, Maozhen},
  journal={IEEE Transactions on Computational Social Systems},
  volume={7},
  number={5},
  pages={1304--1316},
  year={2020},
  publisher={IEEE}
}

@article{sugiyama2007covariate,
  title={Covariate shift adaptation by importance weighted cross validation.},
  author={Sugiyama, Masashi and Krauledat, Matthias and M{\"u}ller, Klaus-Robert},
  journal={Journal of Machine Learning Research},
  volume={8},
  number={5},
  year={2007}
}

@inproceedings{huang2007kmm,
  title={Correcting Sample Selection Bias by Unlabeled Data},
  author={Huang, Jiayuan and Gretton, Arthur and Borgwardt, Karsten M. and Sch{\"o}lkopf, Bernhard and Smola, Alexander J.},
  booktitle={Advances in Neural Information Processing Systems},
  year={2007}
}

\appendix
\renewcommand{\thetheorem}{\arabic{theorem}}
 \setcounter{theorem}{0}

\section{Notation}

\label{app:notation}
Table~\ref{tab:notation} summarizes the notation used throughout this work.
\begin{table}[t]
\centering
\caption{Notation used throughout the paper.}
\label{tab:notation}
\setlength{\tabcolsep}{6pt}
\begin{tabular}{@{}ll@{}}
\toprule
\textbf{Symbol} & \textbf{Meaning} \\
\midrule
$d$ & Ambient dimension (e.g., $784$ for $28\times 28$ MNIST images) \\
$p,q$ & Number of columns (samples or basis vectors) used to form $A$ and $B$ \\
$A\in\mathbb{R}^{d\times p}$, $B\in\mathbb{R}^{d\times q}$ & Dataset matrices (columns are observations in $\mathbb{R}^d$) \\
$a_i$, $b_j$ & Columns of $A$ and $B$ (individual observations) \\
$\col(A)$, $\col(B)$ & Column spaces (subspaces spanned by columns of $A$ and $B$) \\
$x\in\mathbb{R}^p$, $y\in\mathbb{R}^q$ & Coefficients such that $Ax$ and $By$ represent an ambient vector \\
$z\in\mathbb{R}^d$ & Sample / ambient vector under analysis \\
$Ax=By=z$ & Co-span relation (same $z$ representable by both $A$ and $B$) \\
$\mathcal{R}_{\text{co-span}}$ & Relation set $\{(x,y): Ax=By\}$ \\
$\mathcal{R}_{\text{co-span}}(z)$ & Fiber $\{(x,y): Ax=By=z\}$ \\
$\Ker(A)$, $\Ker(B)$ & Nullspaces of $A$ and $B$ \\
$x\perp\Ker(A)$, $y\perp\Ker(B)$ & Minimum-$\ell_2$ coefficient choice among solutions of $Ax=By=z$ \\
$A=HCU$, $B=HSV$ & GSVD form used in the paper \\
$H\in\mathbb{R}^{d\times d}$ & Shared (invertible/left-invertible) ambient frame matrix \\
$U\in\mathbb{R}^{p\times p}$, $V\in\mathbb{R}^{q\times q}$ & Orthogonal factors in the GSVD \\
$C$, $S$ & (Block-)diagonal GSVD factors with $C^\top C + S^\top S = I$ \\
$r,k,t$ & Block sizes in \eqref{eq:cs_blocks} with $r+k+t=d$ \\
$\widetilde C,\widetilde S\in\mathbb{R}^{k\times k}$ & Positive diagonal “shared” blocks in \eqref{eq:cs_blocks} \\
$h_i$ & $i$-th column of $H$ \\
$H^\dagger$ & Moore--Penrose pseudoinverse of $H$ \\
$C^\dagger$, $S^\dagger$ & Moore--Penrose pseudoinverses (often truncated/regularized in practice) \\
$c(z)=H^\dagger z$ & Coordinates of $z$ in the GSVD ambient frame (least-squares if needed) \\
$a(z)=\|C^\dagger c(z)\|_2$,  & GSVD-weighted representation costs for $A$  \\
$b(z)=\|S^\dagger c(z)\|_2$ & GSVD-weighted representation costs for $B$  \\
$\theta(z)\in[0,\pi/2]$ & Alignment angle $\arctan(\|x\|_2/\|y\|_2)=\arctan(a(z)/b(z))$ \\
$\tau$ & Classification threshold on $\theta(z)$ (default $\pi/4$) \\
$e_k$ & $k$-th canonical basis vector \\
$z_{\max},z_{\min}$ & Extreme directions maximizing/minimizing $\theta(z)$ over a constraint set $\mathcal{Z}$ \\
$P,Q$ & Normalized histograms approximating $p(\theta\mid A)$ and $p(\theta\mid B)$ \\
$d_{\mathrm{FR}}(P,Q)$ & Fisher--Rao distance between histograms (Eq.~\eqref{eq:app_fr_bc}) \\
\bottomrule
\end{tabular}
\end{table}


\section{Angle Score and Minimal Classifier Algorithm}
\label{app:algo}
Algorithm~\ref{alg:theta} computes   the GSVD angle score~\eqref{eq:theta_def} and uses it in for classification as described in Section~\ref{sec:uses}.

Matrices $A$ and $B$ are constructed in line~\ref{line:buildAB}, and their GSVD is computed in
line~\ref{line:gsvd}.
For each test sample, projection onto the shared basis is performed in
line~\ref{line:project}, class-specific values are computed in
line~\ref{line:energies}, and the angle score is formed in
line~\ref{line:theta} before thresholding in line~\ref{line:decision}.

\begin{algorithm}[h]
\caption{GSVD angle score and a minimal classifier}
\label{alg:theta}
\begin{algorithmic}[1]
\Require Two labeled datasets $\mathcal{D}_A,\mathcal{D}_B\subset\R^d$, ranks $p,q$, threshold $\tau$
\State Build $A\in\R^{d\times p}$ and $B\in\R^{d\times q}$ from training samples or low-rank bases
\label{line:buildAB}
\State Compute GSVD factors $(H,C,S,U,V)$ such that $A=HCU$ and $B=HSV$
\label{line:gsvd}
\For{each test sample $z\in\R^d$}
    \State $c \leftarrow H^\dagger z$ \Comment{least-squares projection}
    \label{line:project}
    \State $a\leftarrow \norm{C^\dagger c}_2,\;\; b\leftarrow \norm{S^\dagger c}_2$
    \label{line:energies}
    \State $\theta(z)\leftarrow \arctan(a/b)$
    \label{line:theta}
    \State Predict $\hat{\ell}(z)\leftarrow A$ if $\theta(z)\le \tau$, else $B$
    \label{line:decision}
\EndFor
\end{algorithmic}
\end{algorithm}

\begin{remark}\textbf{Note on Dimensionality Reduction.}
Exploiting the structure of $H$ and the characterization of its extreme directions discussed in
Section~\ref{sec:optimization}, a practical micro-optimization for improving classification
efficiency without materially degrading performance is to discard directions of $H$ that receive
comparable weights from $C^\dagger$ and $S^\dagger$.
Intuitively, such directions encode features that are largely shared between the two datasets and
therefore contribute little discriminative information when assessing whether a given instance is
more strongly aligned with $A$ or with $B$.
\end{remark}

\section{Span view: instance alignment and a GSVD-based relative Procrustes}
\label{app:span_view}
This appendix records a complementary interpretation of the same linear relation studied in the main text. While our pipeline is driven by the co-span constraint $Ax=By=z$ (dimension alignment), the same compatibility can be read in a \emph{span} form that emphasizes \emph{instance alignment}.

\subsection{Span parameterization of a linear relation}
A linear relation between two subspaces can be described implicitly (co-span) by
\[
\mathcal{R}_{\text{co-span}}=\{(x,y): Ax = By\},
\]
or explicitly (span) by introducing a shared latent parameter $w$:
\begin{equation}
\mathcal{R}_{\text{span}}
=
\{(x,y): \exists\, w \text{ such that } x = Fw,\; y = Gw\}.
\label{eq:app_span_relation}
\end{equation}
In this view, paired coordinates $(x,y)$ arise from the \emph{same} latent code $w$.
Geometrically, span emphasizes \emph{instance-level pairing}: one latent point generates a compatible point in each domain.

\subsection{Relative Procrustes as instance alignment}
Suppose we have two sets of instance vectors in a shared ambient space (e.g., embeddings)
\[
X = [x_1,\dots,x_n]\in\mathbb{R}^{d\times n},\qquad
Y = [y_1,\dots,y_n]\in\mathbb{R}^{d\times n}.
\]
A classical alignment primitive is (orthogonal) Procrustes:
\begin{equation}
R^\star \in \argmin_{R^\top R = I}\, \|RX - Y\|_F^2.
\label{eq:app_procrustes}
\end{equation}
In the dataset-comparison setting we often do not want to align in the raw ambient basis, but rather in a \emph{relative geometry} that discounts directions that are not shared and emphasizes directions that are shared.

\subsection{GSVD provides the natural coordinate system}
Let the two datasets define matrices $A$ and $B$ and consider their GSVD
\[
A = H C U,\qquad B = H S V,\qquad C^\top C + S^\top S = I.
\]
The shared frame $H$ defines coordinates for any $z\in\mathbb{R}^d$ via $c=H^\dagger z$.
Thus instance sets admit GSVD-frame representations
\[
\widetilde X = H^\dagger X,\qquad \widetilde Y = H^\dagger Y.
\]
Instance alignment can then be performed in this frame, rather than in the raw ambient basis.

\subsection{Angle-based truncation for relative alignment}
The diagonal structure of $(C,S)$ suggests a truncation rule: directions where both factors are non-negligible correspond to shared structure.
Let $\mathcal{I}_{\text{shared}}$ denote indices deemed shared (e.g., by thresholding $\min\{c_{ii},s_{ii}\}$) and define
\[
P_{\text{shared}} = \mathrm{diag}(p_1,\dots,p_d),\quad
p_i = \mathbf{1}\{i\in \mathcal{I}_{\text{shared}}\}.
\]
Then truncated coordinates are $\widetilde X_{\text{sh}} = P_{\text{shared}}\widetilde X$ and $\widetilde Y_{\text{sh}} = P_{\text{shared}}\widetilde Y$, and a GSVD-aware Procrustes alignment is
\begin{equation}
R^\star_{\text{sh}}
\in
\argmin_{R^\top R = I}\, \|R\widetilde X_{\text{sh}} - \widetilde Y_{\text{sh}}\|_F^2.
\label{eq:app_relative_procrustes}
\end{equation}
This objective focuses alignment on the shared geometry, ignoring weakly shared directions.

\section{Analytical Properties and Proofs for the GSVD-Based Framework}
\label{sec:appanalytproof}
We begin with some initial definitions. Let $\tilde{x}=(\tilde x_1, 
\tilde x_2, 
\tilde x_3)$ and $\tilde{y}=(\tilde y_1, 
\tilde y_2, 
\tilde y_3)$ where the dimensions correspond to the blocks of $C$ and $S$, 
\[
C \tilde x
=
\begin{bmatrix}
I_{r} & 0 & 0 \\
0 & \widetilde C & 0 \\
0 & 0 & 0_{t}
\end{bmatrix}
\begin{bmatrix}
\tilde x_1\\
\tilde x_2\\
\tilde x_3
\end{bmatrix}
=
\begin{bmatrix}
\tilde x_1\\
\widetilde C\,\tilde x_2\\
0
\end{bmatrix},
\]

\[
S \tilde y
=
\begin{bmatrix}
0_{r} & 0 & 0 \\
0 & \widetilde S & 0 \\
0 & 0 & I_{t}
\end{bmatrix}
\begin{bmatrix}
\tilde y_1\\
\tilde y_2\\
\tilde y_3
\end{bmatrix}
=
\begin{bmatrix}
0\\
\widetilde S\,\tilde y_2\\
\tilde y_3
\end{bmatrix},
\]
where $\tilde{x}_1 \in \mathbb{R}^{r}$, $\tilde{x}_2 \in \mathbb{R}^{k}$ and $\tilde{x}_3 \in \mathbb{R}^{t}$. Similarly,  $\tilde{y}_1 \in \mathbb{R}^{r}$, $\tilde{y}_2 \in \mathbb{R}^{k}$ and $\tilde{y}_3 \in \mathbb{R}^{t}$. 

If \(C\tilde x = S\tilde y\), then componentwise
\[
\tilde x_1 = 0, 
\qquad
\widetilde C\,\tilde x_2 = \widetilde S\,\tilde y_2,
\qquad
\tilde y_3 = 0.
\]

In what follows, we let
  \begin{equation}
    \tilde x := Ux, \quad  \tilde y := Vy.
\end{equation}
Then, our goal in this appendix is twofold. First, we  show that
\begin{equation}
   Ax=By  \Longleftrightarrow C \tilde{x} = S\tilde{y}.
\end{equation}
with
\begin{equation}
\tilde x
=
\begin{bmatrix}
0\\
\widetilde S w\\
\tilde x_3
\end{bmatrix},
\qquad
\tilde y
=
\begin{bmatrix}
\tilde y_1\\
\widetilde C w\\
0
\end{bmatrix},
\end{equation}
where $\tilde{x}_3$ and $\tilde{y}_1$ are free and $  w\in\mathbb{R}^{k}$.

Second, we show that if $x \perp \Ker(A)$ and $y \perp \Ker(B)$ then each  solution $(x,y)$ of $ Ax=By$ uniquely characterizes a pair $(\tilde{x},\tilde{y})$,
\begin{equation}
\tilde x
=
\begin{bmatrix}
0\\
\widetilde S w\\
0
\end{bmatrix},
\qquad
\tilde y
=
\begin{bmatrix}
0\\
\widetilde C w\\
0
\end{bmatrix}.
\end{equation}


\begin{lemma}[GSVD induces an explicit parameterization of the co-span constraint]
\label{lem:implicit-explicit}
Under Definition~\ref{eq:intro_gsvd},
\[
Ax=By \Longleftrightarrow  C \tilde{x} = S\tilde{y}
\]
where
\[ 
\begin{bmatrix}
\tilde x_1\\
\widetilde C\,\tilde x_2\\
0
\end{bmatrix} = \begin{bmatrix}
0\\
\widetilde S\,\tilde y_2\\
\tilde{y}_3
\end{bmatrix}
\]
and $\tilde x_3$ and $\tilde y_1$ are  free. Moreover, there exists a vector $w$ such that
\[
\tilde x =
\begin{bmatrix}
0\\
\widetilde S\, w\\
\tilde x_3
\end{bmatrix},
\qquad
\tilde y =
\begin{bmatrix}
\tilde y_1\\
\widetilde C\, w\\
0
\end{bmatrix},
\]
with $\tilde x_3$ and $\tilde y_1$ free.\end{lemma}

\begin{proof}
The proof proceeds as the following sequence of equivalence steps:
\begin{calculation}[\Longleftrightarrow]
Ax = By
\step{GSVD definition}
HCUx = HSVy.
\step{By GSVD definition  $\exists \bar H$ such that $\bar HH = I$ }
CUx = SVy
\step{Change of Variables: $\tilde x = Ux$ and $\tilde y=Vy$}
C\tilde x = S \tilde y
\step{We can expand the $C,S$ matrices by the definition in \eqref{def:gsvd_block} to create:}
\begin{bmatrix}
I_{r} & 0 & 0 \\
0 & \widetilde C & 0 \\
0 & 0 & 0_{t}
\end{bmatrix}
\begin{bmatrix}
\tilde x_1\\
\tilde x_2\\
\tilde x_3
\end{bmatrix} = \begin{bmatrix}
0_{r} & 0 & 0 \\
0 & \widetilde S & 0 \\
0 & 0 & I_{t}
\end{bmatrix}
\begin{bmatrix}
\tilde y_1\\
\tilde y_2\\
\tilde y_3
\end{bmatrix}
\step{Algebra}
\begin{bmatrix}
\tilde x_1\\
\widetilde C\,\tilde x_2\\
0
\end{bmatrix} = \begin{bmatrix}
0\\
\widetilde S\,\tilde y_2\\
\tilde{y}_3
\end{bmatrix}
\text{with $\tilde x_3$ and $\tilde y_1$ free.}
\end{calculation}

We now focus on the middle block. For each index \(i\), the relation
\[
\widetilde C_i\,\tilde x_{2i}=\widetilde S_i\,\tilde y_{2i}
\]
holds. Since \(\widetilde C_i>0\) and \(\widetilde S_i>0\), this equality can be reparameterized by introducing a scalar \(w_i\) such that
\[
\tilde x_{2i}=\widetilde S_i\,w_i,
\qquad
\tilde y_{2i}=\widetilde C_i\,w_i.
\]
Because \(\widetilde C\) and \(\widetilde S\) are diagonal, this construction applies entrywise. Collecting the coordinates yields the vector form
\[
\tilde x_2=\widetilde S\,w,
\qquad
\tilde y_2=\widetilde C\,w,
\qquad
w\in\mathbb{R}^k.
\]

To find the values of $\tilde x_1, \tilde y_3$ we can look at the equality restriction above.

Finally, the explicit parametrization of $\tilde x$ and $\tilde y$ is: 
\[
\tilde x =
\begin{bmatrix}
0\\
\widetilde S\, w\\
\tilde x_3
\end{bmatrix},
\qquad
\tilde y =
\begin{bmatrix}
\tilde y_1\\
\widetilde C\, w\\
0
\end{bmatrix},
\]

\end{proof}

\begin{lemma}[Nullspace orthogonality fixes the free GSVD blocks]
\label{lem:null-orth}
Under Definition~\ref{eq:intro_gsvd},
\[
x \perp \Ker(A)
\quad\Longleftrightarrow\quad
\tilde x_3 = 0.
\]
\[
y \perp \Ker(B)
\quad\Longleftrightarrow\quad
\tilde y_1 = 0.
\]
\end{lemma}

\begin{proof}
We show the claim for $A$; the proof for $B$ is analogous.
\begin{calculation}[\Longleftrightarrow]
x \perp \Ker(A)
\step{GSVD definition}
x \perp \Ker(HCU)
\step{By GSVD definition  $\exists \bar H$ such that $\bar HH = I$} 
x \perp \Ker(CU)
\step{Perpendicularity to kernel definition}
\langle x,k\rangle=0, \space \forall k \text{ s.t. } CUk=0
\step{Change of variable: $Uk=\tilde k$, so $k=U^T\tilde k$}
\langle x,U^T\tilde k\rangle=0, \space \forall \tilde k \text{ s.t. } C\tilde k=0
\step{Duality of inner product}
\langle Ux,\tilde k\rangle=0, \space \forall \tilde k \text{ s.t. } C\tilde k=0
\step{$\tilde x$ definition}
\langle \tilde x ,\tilde k\rangle=0, \space \forall \tilde k \text{ s.t. } C\tilde k=0
\end{calculation}

Expanding the definition of C we have
\[
\begin{bmatrix}
I_{r} & 0 & 0 \\
0 & \widetilde C & 0 \\
0 & 0 & 0_{t}
\end{bmatrix}
\begin{bmatrix}
\tilde k_1\\
\tilde k_2\\
\tilde k_3
\end{bmatrix} = \begin{bmatrix}
0\\
0\\
0
\end{bmatrix}
\]
which forces $\tilde k_1=0$ and $\tilde k_2=0$ while keeping $\tilde k_3$ unconstrained.
Therefore, to guarantee that $\langle x,k \rangle = 0$, we must enforce $\tilde x_3=0$ and leave $\tilde x_1$ and  $\tilde x_2$ unconstrained.
\end{proof}

\begin{theorem}[GSVD produces alignment angles]   
\label{alg:theta_proof}  Given datasets $A$ and $B$, and an observation $z$, the alignment angle  between $A$ and $B$ with respect to $z$ (cf. Definition~\ref{eq:keydefinition}) is given by
\begin{equation}
\theta(z)=\arctan\!\left(\frac{a(z)}{b(z)}\right)
=
\arctan\!\left(
\frac{\norm{C^\dagger\,c(z)}_2}{\norm{S^\dagger\,c(z)}_2}
\right),
\label{eq:theta_gsvd}
\end{equation}
where 
\begin{equation}
c(z)=H^\dagger z, 
\end{equation}
and  $(H,C,S,U,V)$ are  the GSVD factors in~\eqref{eq:intro_gsvd}.
\end{theorem}

\begin{proof}
For a given vector $z \in \mathbb{R}^d$, we express $z$ in the shared GSVD frame
by computing its coordinates with respect to $\operatorname{col}(H)$:
\begin{equation}
\argmin_{c \in \mathbb{R}^d}
\;\|Hc - z\|_2^2,
\qquad\text{so that}\qquad
c = H^\dagger z.
\end{equation}
This least-squares projection provides the natural notion of GSVD-frame coordinates
when $z$ is noisy or does not lie exactly in $\operatorname{col}(H)$; when
$z \in \operatorname{col}(H)$, the equality $z = Hc$ holds exactly. Under the GSVD parametrization, the co-span constraint $Ax = By = z$ implies the
following chain of equivalences:
{\setlength{\jot}{0.5pt}%
 \setlength{\abovedisplayskip}{1pt}%
 \setlength{\belowdisplayskip}{1pt}%
 \setlength{\abovedisplayshortskip}{0pt}%
 \setlength{\belowdisplayshortskip}{0pt}%
\begin{calculation}[\Longleftrightarrow]
Ax = By = z
\step{$A=HCU,\ B=HSV,\ z=Hc$}
(HCU)x = (HSV)y = Hc
\step{By GSVD definition $\exists\bar H$ such that $\bar HH=I$}
CUx = SVy = c
\step{Change of variables: $\tilde x:=Ux,\ \tilde y:=Vy$}
C\tilde x = S\tilde y = c
\step{By Lemma  (~\ref{lem:implicit-explicit})}
\begin{bmatrix}
\tilde x_1\\
\widetilde C\,\tilde x_2\\
0
\end{bmatrix} = \begin{bmatrix}
0\\
\widetilde S\,\tilde y_2\\
\tilde{y}_3
\end{bmatrix} = \begin{bmatrix}c_1 \\ c_2 \\ c_3\end{bmatrix} \qquad \text{with $\tilde x_3$ and $\tilde y_1$ free}
\step{The equality tells us that $\tilde{x}_1 = c_1 =0$ and $\tilde y_3 = c_3 =  0$}
\begin{bmatrix}
0\\
\widetilde C\,\tilde x_2\\
0
\end{bmatrix} = \begin{bmatrix}
0\\
\widetilde S\,\tilde y_2\\
0
\end{bmatrix} = \begin{bmatrix}0 \\ c_2 \\ 0\end{bmatrix} \qquad \text{with $\tilde x_3$ and $\tilde y_1$ free}
\step{Using Lemma ~\ref{lem:null-orth}}
\begin{bmatrix}
0\\
\widetilde C\,\tilde x_2\\
0
\end{bmatrix} = \begin{bmatrix}
0\\
\widetilde S\,\tilde y_2\\
0
\end{bmatrix} = \begin{bmatrix}0 \\ c_2 \\ 0\end{bmatrix} 
\step{Focus on the middle term}
\widetilde C\,\tilde x_2 = \widetilde S\,\tilde y_2 =c_2 
\step{Isolating expressions}
\widetilde C\,\tilde x_2 = c_2 \qquad \widetilde S\,\tilde y_2 =c_2
\step{Solving for $\tilde x_2$ and $\tilde y_2$}
\tilde x_2=\widetilde C^{-1}c_2,\quad \tilde y_2=\widetilde S^{-1}c_2
\step{Now consider the ratio expression}
\frac{\|x\|_2}{\|y\|_2}
\step[=]{$U,V$ orthogonal}
\frac{\|\tilde x\|_2}{\|\tilde y\|_2}
\step[=]{Norm of middle term}
\frac{\|\widetilde C^{-1}c_2\|_2}{\|\widetilde S^{-1}c_2\|_2}
\step[=]{Accounting for all blocks}
\frac{\|C^{\dagger}c\|_2}{\|\widetilde S^{\dagger}c\|_2}
\end{calculation}
Because the boundary entries of \(c\) are constrained to be zero, the resulting system admits multiple matrices that satisfy the final condition. To ensure a unique solution, we therefore employ the Moore–Penrose pseudoinverse.
} \end{proof}

Next, we consider the problem of finding extreme directions. 
Recall that the direction  $z$ that  maximizes the alignment angle $\theta(z)$ is given by Theorem~\ref{thm:ratio-max},
\[
\argmax_{x,y}\ \arctan \left(\frac{\|x\|_2^2}{\|y\|_2^2}\right)
\quad\text{s.t.}\quad
Ax=By,\;\;  (x,y) \in\mathcal{R}_{\text{co-span}}(z), \;\; x\perp \Ker(A),\;\; y\perp \Ker(B).
\]

\begin{theorem}[GSVD produces extreme directions]
Given datasets $A$ and $B$, the extreme directions $z_{\textrm{max}}$ and $z_{\textrm{min}}$
that maximize and minimize the alignment angle $\theta(z)$, respectively, are unique and given by
$z_{\textrm{max}}=Ax_{\textrm{max}}=By_{\textrm{max}}$ and
$z_{\textrm{min}}=Ax_{\textrm{min}}=By_{\textrm{min}}$, where
$z_{\textrm{max}} = h_{r+k}$, $z_{\textrm{min}} = h_{r+1}$ and
\[
x_{\textrm{max}} =
U^T\begin{bmatrix}
0\\
\widetilde S\widetilde C^{-1}e_{k}\\
0
\end{bmatrix},
\quad
y_{\textrm{max}} =
V^T\begin{bmatrix}
0\\
e_{k}\\
0
\end{bmatrix},
\quad
x_{\textrm{min}} =
U^T\begin{bmatrix}
0\\
\widetilde C\widetilde S^{-1}e_{1}\\
0
\end{bmatrix},
\quad
y_{\textrm{min}} =
V^T\begin{bmatrix}
0\\
 e_{1}\\
0
\end{bmatrix}.
\]
$(H,C,S,U,V)$ are the GSVD factors in~\eqref{eq:intro_gsvd}, $e_i$ is the $i$-th canonical basis vector,
i.e., $(e_i)_j = 1$ if $j=i$ and $(e_i)_j = 0$ otherwise, and $h_i$ is the $i$-th column of matrix $H$.
\end{theorem}

\begin{proof}
Since $\arctan$ is strictly monotonic increasing, we can focus on finding $\argmax_{x,y}\  \frac{\|x\|_2^2}{\|y\|_2^2}$ with the same restraints.

By Lemma~\ref{lem:implicit-explicit}, $Ax=By$ implies
$\tilde x_1=0$, $\tilde y_3=0$, and $\tilde x_2=\widetilde S w$, $\tilde y_2=\widetilde C w$ for some $w$,
with $\tilde x_3$ and $\tilde y_1$ free.

By Lemma~\ref{lem:null-orth}, the additional constraints $x\perp\Ker(A)$ and $y\perp\Ker(B)$ force
$\tilde x_3=0$ and $\tilde y_1=0$.

Thus only the middle blocks remain, i.e., $\tilde x=(0,\widetilde S w,0)$ and $\tilde y=(0,\widetilde C w,0)$.

Since $U$ and $V$ are orthogonal, $\|x\|_2=\|\tilde x\|_2$ and $\|y\|_2=\|\tilde y\|_2$.
Therefore
\[
\frac{\|x\|_2^2}{\|y\|_2^2}=\frac{\|\widetilde S w\|_2^2}{\|\widetilde C w\|_2^2}.
\]
As defined in \eqref{def:gsvd_block} the diagonal entries of the middle matrices are strictly positive, so they have inverses. Now we denote $\tilde w=\widetilde C w$ and thus $w =\widetilde C^{-1}\tilde w$:
\[
\frac{\|\widetilde S \widetilde C^{-1}\tilde w\|_2^2}{\|\widetilde C \widetilde C^{-1}\tilde w\|_2^2} = 
\frac{\|\widetilde S \widetilde C^{-1}\tilde w\|_2^2}{\|\tilde w\|_2^2} 
\]

The ratio becomes a Rayleigh quotient for the diagonal operator $\widetilde S\widetilde C^{-1}$, hence is maximized by concentrating all mass on the index of its largest diagonal entry, i.e., $\tilde w=e_k$.

Reconstruction of $(x,y)$ follows from $\tilde x=Ux$ and $\tilde y=Vy$.

Finally, to uncover the desired direction we use substitution on $Ax=By$:
\[
By = HSV(V^\top \tilde y)=HS\tilde y
=H\begin{bmatrix}0\\ \widetilde S e_k\\ 0\end{bmatrix} = \begin{bmatrix}\widetilde S e_k\\ 0\end{bmatrix} = \tilde s_k H\begin{bmatrix}0\\  e_k\\ 0\end{bmatrix} = \tilde s_k\cdot h_{r+k}
\]
and
\[
Ax = HCU(U^\top \tilde x)=HC\tilde x
=H\begin{bmatrix}0\\ \widetilde C(\widetilde S\widetilde C^{-1} e_r)\\ 0\end{bmatrix}
=H\begin{bmatrix}0\\ \widetilde S e_k\\ 0\end{bmatrix} = \tilde c_k\tilde s_k \frac{1}{\tilde c_k}H\begin{bmatrix}0\\  e_k\\ 0\end{bmatrix} = \tilde s_k\cdot h_{r+k}
\]
Note that as $\widetilde S, \widetilde C, \widetilde C^{-1}$ are diagonal and $e_k$ is a canonical basis vector, the operation is simply a multiplication by the scalars $\tilde s_k, \tilde c_k, \frac{1}{\tilde c_k}$.

We can then normalize the vector $h_{r+k}$.
\end{proof}
\begin{remark} \textbf{Practical notes on numerical stability. }
In practice, we compute $C^\dagger$ and $S^\dagger$ by inverting diagonal entries above a tolerance (or via truncated SVD / Tikhonov regularization). We also recommend reporting sensitivity of $\theta$ to rank choices $(p,q)$ and preprocessing (centering/normalization).
\end{remark}


\section{MNIST Fashion: angle histograms under the same GSVD setup}
\label{app:fashion_mnist}

We replicate the MNIST-digit protocol of Section~\ref{sec:exp} on Fashion-MNIST, keeping the
\emph{same preprocessing} (mean centering in the ambient pixel space), the \emph{same GSVD pipeline},
and the \emph{same rank choices} $(p,q)$ and train/test split logic used in the digit experiments. As before, we construct $A\in\mathbb{R}^{d\times p}$ and $B\in\mathbb{R}^{d\times q}$ by stacking
(class-specific) training images as columns (after preprocessing), compute the GSVD factors, and evaluate
$\theta(z)$ for each test sample $z$ from the two classes.

Figure~\ref{fig:fashion_mnist_examples} illustrates representative Fashion-MNIST samples used in our GSVD pipeline, showing both raw images and their corresponding vectors after mean-centering and vectorization.

Figure~\ref{fig:theta_distributions_fashion} reports empirical histograms of $\theta$ for two representative class pairs. 
In Fig.~\ref{fig:theta_distributions_fashion}(a) (\texttt{T-shirt/top} vs.\ \texttt{Sneaker}), the two distributions are well separated, indicating limited shared structure between upper-body garments and footwear under this linear subspace view.
In contrast, Fig.~\ref{fig:theta_distributions_fashion}(b) (\texttt{Sneaker} vs.\ \texttt{Ankle boot}) exhibits substantial overlap, consistent with the fact that both classes correspond to visually related types of shoes and therefore share more dominant directions.

These results qualitatively mirror the behavior observed on MNIST digits: class pairs with stronger visual or semantic similarity tend to produce overlapping angle distributions, whereas more distinct pairs yield clearer separation. 
Importantly, no dataset-specific tuning is introduced for Fashion-MNIST, suggesting that the angle score reflects intrinsic geometric relations between class-conditioned subspaces rather than artifacts of hyperparameter choices.

\begin{figure}[h]
\centering
\setlength{\tabcolsep}{2pt}
\begin{tabular}{cccc}
\safeincludegraphics[height=2.4cm]{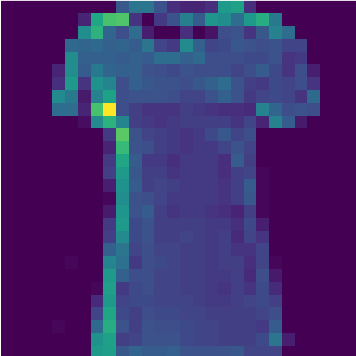} &
\safeincludegraphics[height=2.4cm]{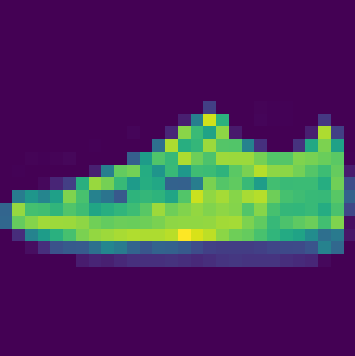} &
\safeincludegraphics[height=2.4cm]{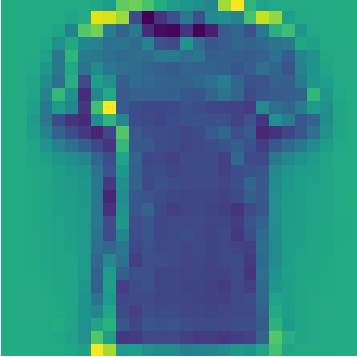} &
\safeincludegraphics[height=2.4cm]{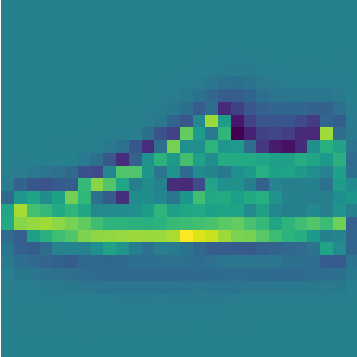} \\
(a) & (b) & (c) & (d)
\end{tabular}
\caption{Representative MNIST Fashion samples for the GSVD pipeline. (a,b) raw ``\texttt{T-shirt}'' and ``\texttt{Sneakers}''; (c,d) corresponding vectors after mean-centering difference.}
\label{fig:fashion_mnist_examples}
\end{figure}

\begin{figure}[H]
\centering
\begin{tabular}{cc}
\subfloat[]{
    \includegraphics[width=0.45\textwidth]{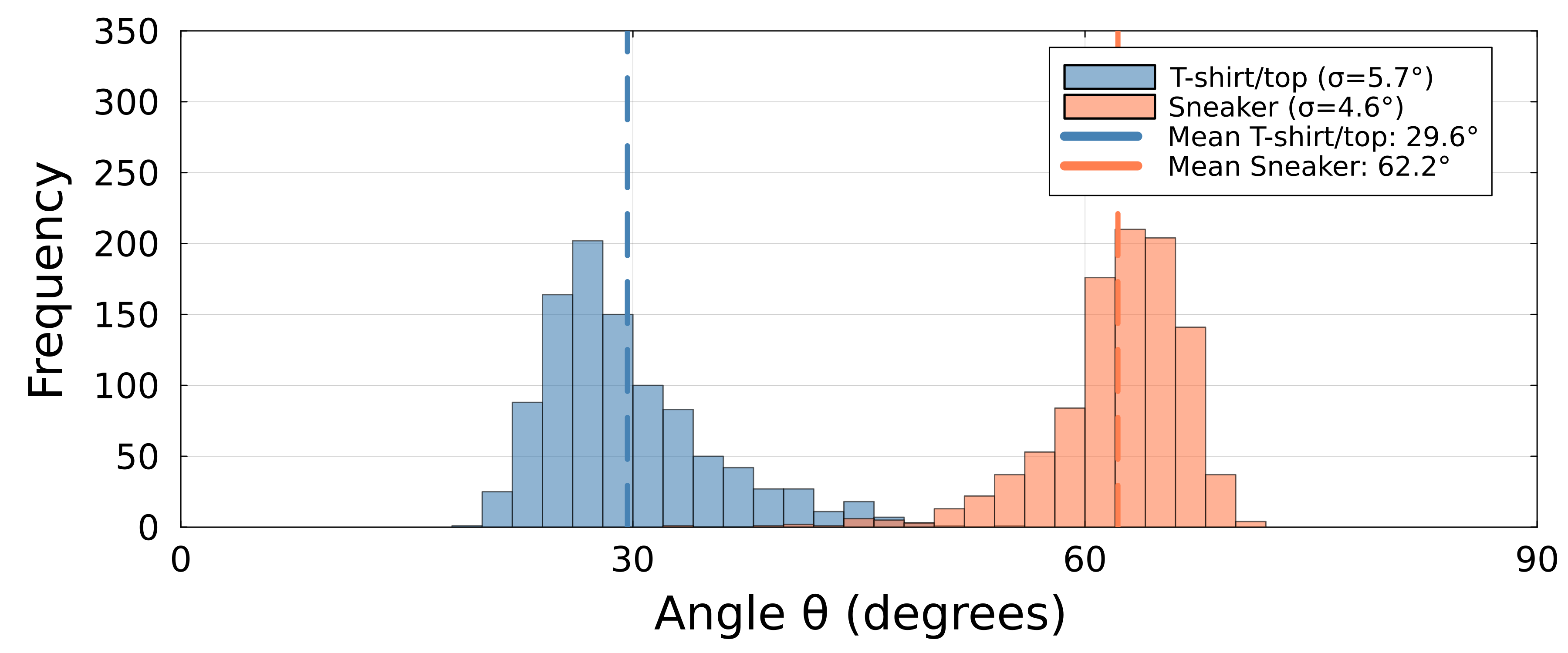}
}
&
\subfloat[]{
    \includegraphics[width=0.45\textwidth]{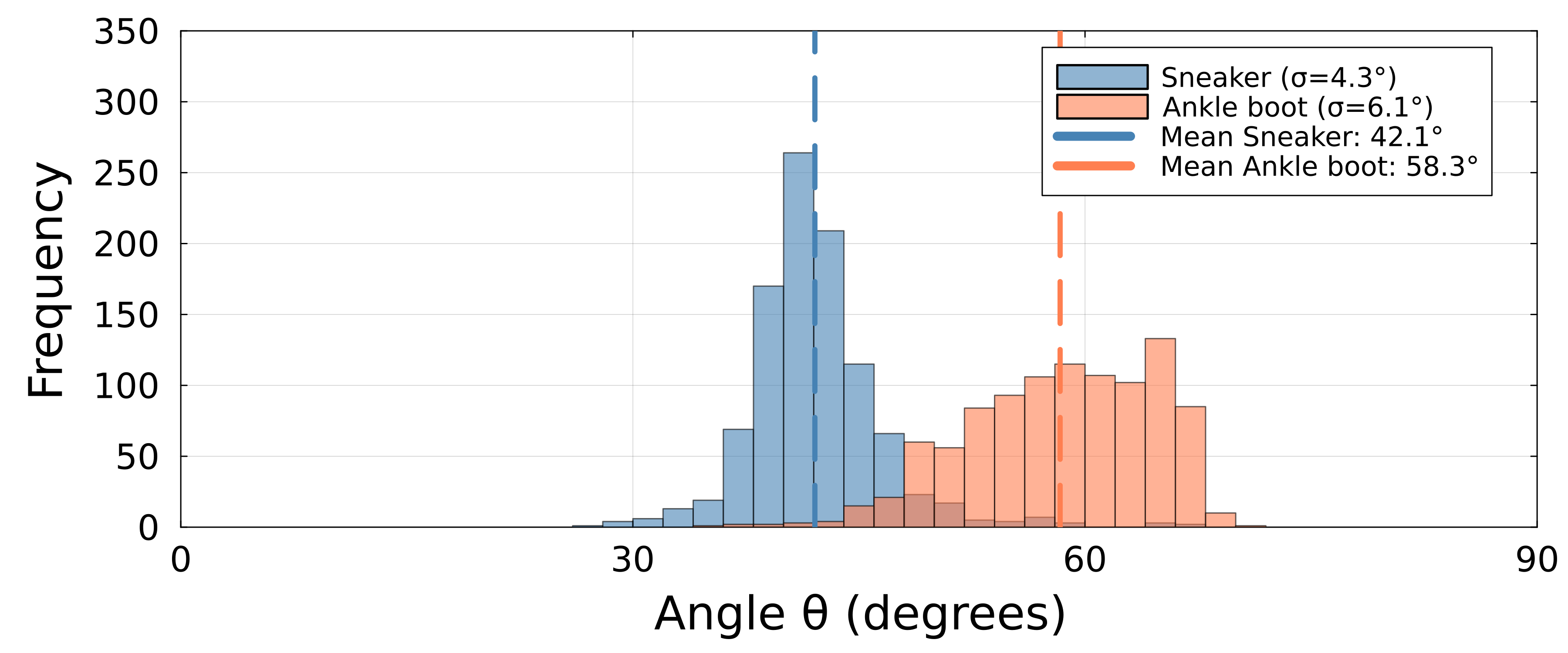}
}
\end{tabular}
\caption{Empirical distributions of the angle $\theta$ on the test set for two Fashion-MNIST class pairs under the same GSVD pipeline used for MNIST digits (same preprocessing and the same rank choices $(p,q)$).
(a) \texttt{T-shirt/top} (class 0) vs.\ \texttt{Sneaker} (class 7), where $A$ is built from \texttt{T-shirt/top} columns and $B$ from \texttt{Sneaker} columns; the separation indicates limited shared structure.
(b) \texttt{Sneaker} (class 7) vs.\ \texttt{Ankle boot} (class 9), which shows larger overlap, consistent with more shared directions between these visually related footwear classes.
As in the main text, angles near $0$ indicate stronger alignment with $A$ and angles near $90^\circ$ indicate stronger alignment with $B$.}
\label{fig:theta_distributions_fashion}
\end{figure}

\section{Global separation via Fisher--Rao distance}
\label{sec:fisher_rao}

Next, we report additional results  on the separation via Fisher--Rao distance presented in Section~\ref{sec:exp}.  

The per-sample angle $\theta(z)$ characterizes local alignment, while the class-conditional distributions $p(\theta \mid A)$ and $p(\theta \mid B)$ capture \emph{global} geometric separation between datasets. We quantify this separation by measuring the discrepancy between the two $\theta$-histograms with the Fisher–Rao distance on the probability simplex.

Table~\ref{tab:fr_mnist_pairs} reports Fisher–Rao distances (in radians) between class-conditional $\theta$-histograms for several MNIST digit pairs, using 30 uniformly spaced bins over $[0^\circ,90^\circ]$. Larger Fisher–Rao distances indicate less histogram overlap and thus stronger geometric separability.

The results agree with visual intuition and the underlying angle histograms. Pairs such as $(0,7)$ and $(1,5)$ achieve the largest distances, reflecting minimal overlap between their $\theta$ distributions and showing that each digit is, on average, represented much more efficiently by its own class-specific subspace than by the other class’s. In contrast, $(4,9)$ yields a much smaller distance, indicating greater overlap and hence higher geometric ambiguity, consistent with the well-known visual similarity between handwritten 4s and 9s.

\begin{table}[h]
\centering
\caption{Fisher--Rao distance (in radians) between the class-conditional $\theta$ histograms for selected MNIST digit pairs (30 bins over $[0^\circ,90^\circ]$).}
\setlength{\tabcolsep}{10pt}
\begin{tabular}{@{}cc@{}}
\toprule
Digit pair $(A,B)$ & Fisher--Rao distance (rad) \\
\midrule
$(0,7)$ & 2.773465 \\
$(1,5)$ & 2.731825 \\
$(3,9)$ & 2.665724 \\
$(4,9)$ & 2.336575 \\
\bottomrule
\end{tabular}
\label{tab:fr_mnist_pairs}
\end{table}

Figure~\ref{fig:posteriori_distributions} illustrates the posterior uncertainty described in Section~\ref{sec:fisherraodist}.
Using 1600 images across both digits, the x-axis is discretized into equally spaced angle bins, while the y-axis represents the posterior class probabilities defined in \eqref{eq:posty}. As developed throughout this paper, the posterior distributions reflect the principle that smaller alignment angles are associated with instances from $A$ while larger angles are associated with instances from $B$ intermediate angles capture regions of mutual coherence between the two classes. In panels (a), (b), and (d), this expected structure is clearly observed, with dominant color regions corresponding to their respective angle ranges, even though certain bins are not populated by samples in the test set. In panel (c), an outlier is observed: a single instance from class$B$ is assigned a low alignment angle. Aside from this isolated case, the distribution conforms to the established interpretation.

\begin{figure}[h]
\centering
\begin{subfigure}[t]{0.48\textwidth}
  \centering
  \safeincludegraphics[width=\linewidth]{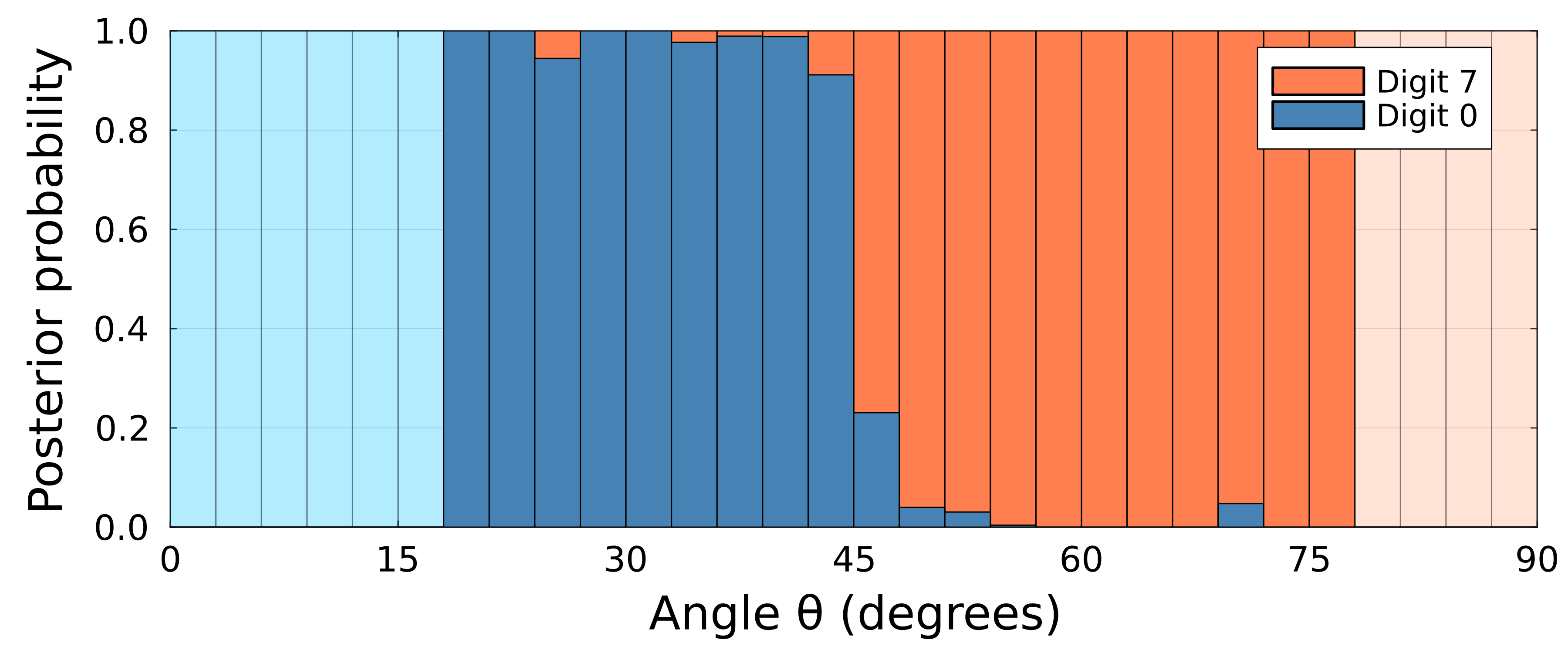}
  \caption{Digits ``0'' vs ``7''}
\end{subfigure}\hfill
\begin{subfigure}[t]{0.48\textwidth}
  \centering
  \safeincludegraphics[width=\linewidth]{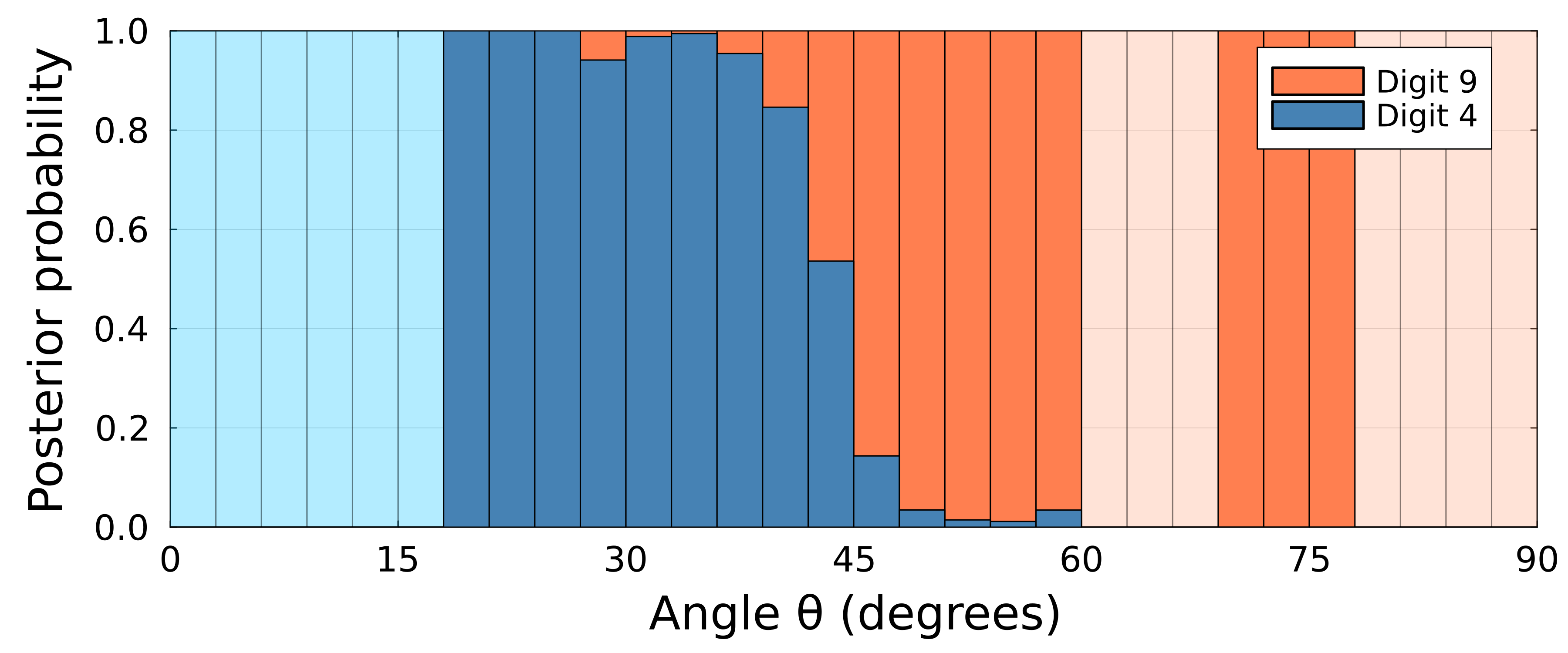}
  \caption{Digits ``4'' vs ``9''}
\end{subfigure}

\vspace{0.6em}

\begin{subfigure}[t]{0.48\textwidth}
  \centering
  \safeincludegraphics[width=\linewidth]{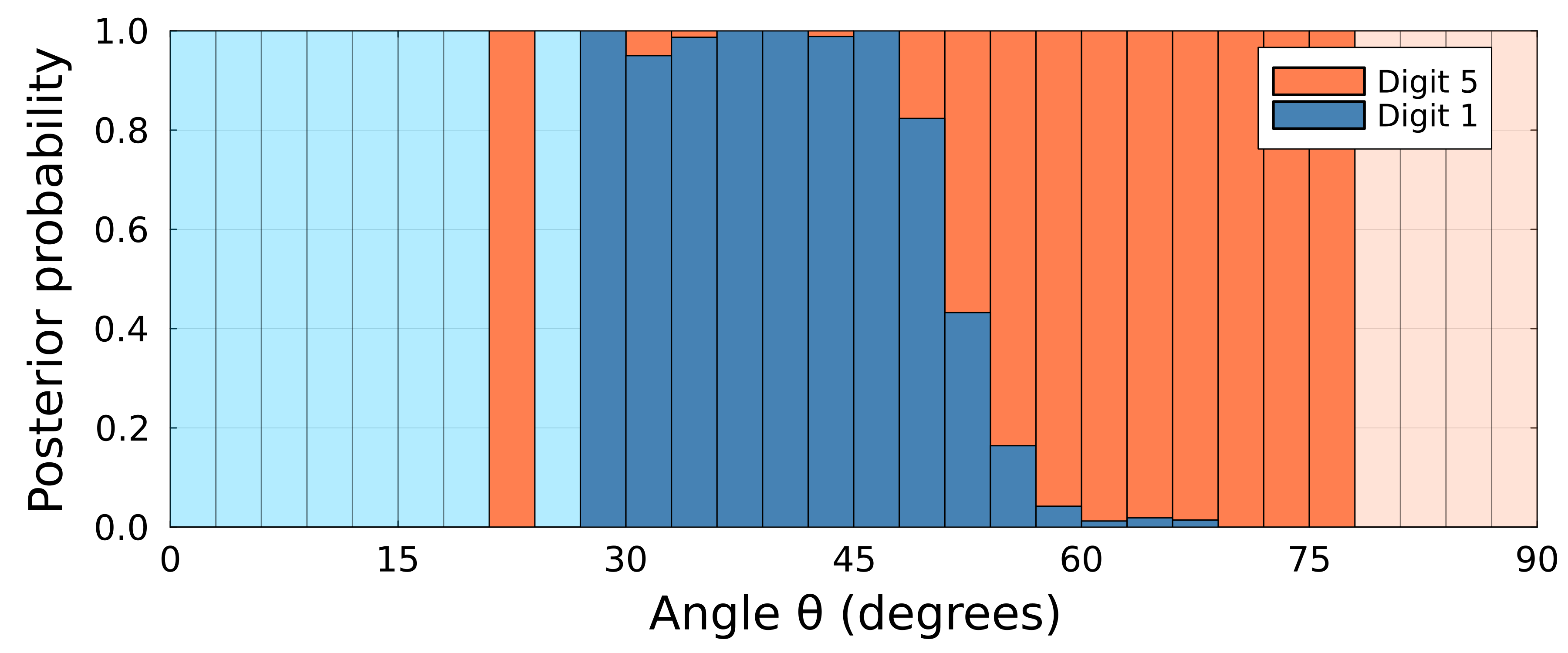}
  \caption{Digits ``1'' vs ``5''}
\end{subfigure}\hfill
\begin{subfigure}[t]{0.48\textwidth}
  \centering
  \safeincludegraphics[width=\linewidth]{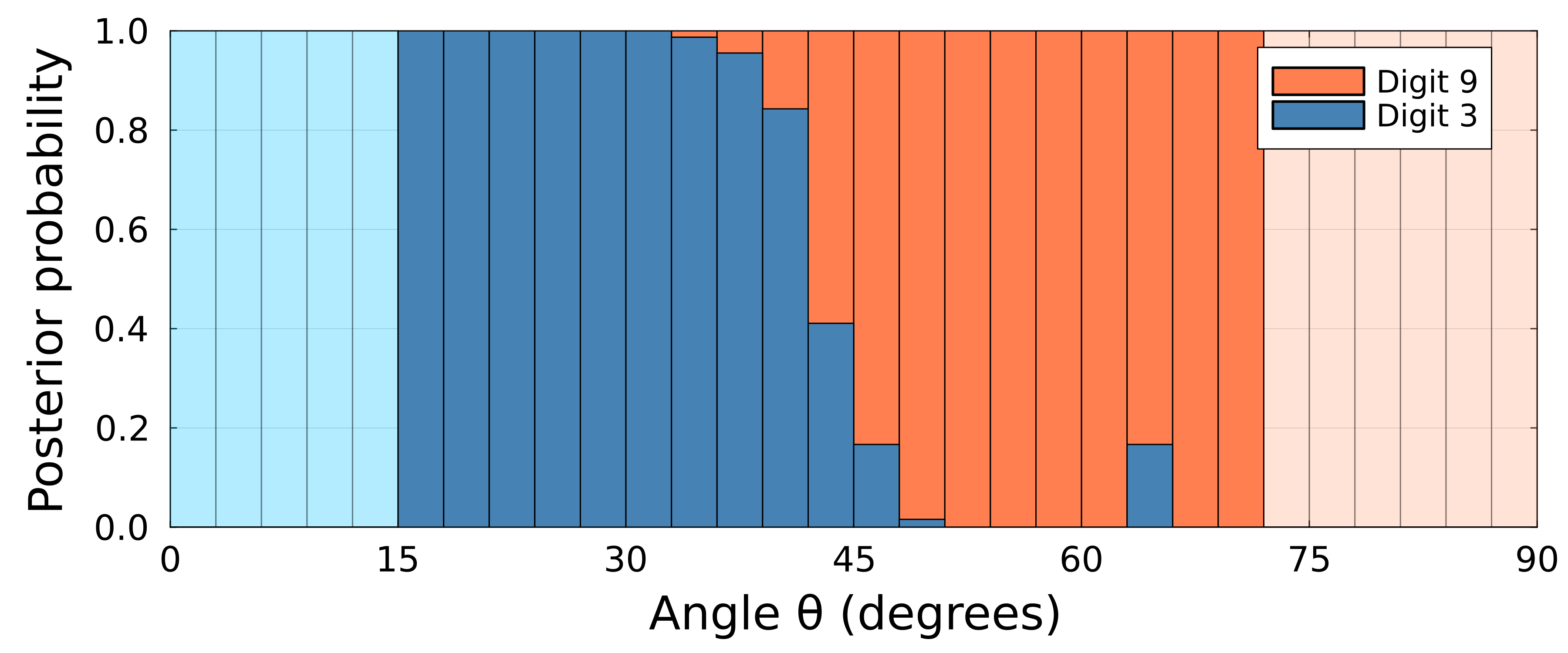}
  \caption{Digits ``3'' vs ``9''}
\end{subfigure}

\caption{Posterior distributions across discretized angle bins for four digit pairs in the MNIST test set.
Angles near $0$ reflect stronger alignment with the first digit in the pair, whereas angles near $90$ reflect stronger alignment with the second. Translucent bins indicate angle ranges for which no samples were observed.}
\label{fig:posteriori_distributions}
\end{figure}

\section{Hellinger vs.\ Fisher--Rao (Bhattacharyya) geometry}
\label{app:hellinger_fr}

This appendix relates the (squared) Hellinger distance between discrete distributions to the
Fisher--Rao distance on the probability simplex. The key bridge is the
\emph{Bhattacharyya coefficient}, i.e., the inner product of square-root densities.
 
Let $P=(P_i)_{i=1}^m$ and $Q=(Q_i)_{i=1}^m$ be probability vectors on $m$ bins, with
$P_i\ge 0$, $Q_i\ge 0$, and $\sum_i P_i=\sum_i Q_i=1$. Define the square-root embedding
\begin{equation}
\psi(P) := \big(\sqrt{P_1},\ldots,\sqrt{P_m}\big)\in \mathbb{R}^m,
\qquad \|\psi(P)\|_2=1,
\label{eq:app_sqrt_embedding}
\end{equation}
and similarly $\psi(Q)$. Recall that the  \emph{Bhattacharyya coefficient} is
\begin{equation}
BC(P,Q) := \langle \psi(P),\psi(Q)\rangle = \sum_{i=1}^m \sqrt{P_iQ_i}\in[0,1].
\label{eq:app_bc}
\end{equation}

\subsection{Hellinger distance as Euclidean distance in the square-root embedding}
The (squared) Hellinger distance is
\begin{equation}
H^2(P,Q)
:= \frac12 \sum_{i=1}^m \big(\sqrt{P_i}-\sqrt{Q_i}\big)^2
= \frac12 \|\psi(P)-\psi(Q)\|_2^2.
\label{eq:app_hellinger_def}
\end{equation}
Expanding the squared norm and using $\sum_i P_i=\sum_i Q_i=1$ gives
\begin{align}
\|\psi(P)-\psi(Q)\|_2^2
&= \|\psi(P)\|_2^2 + \|\psi(Q)\|_2^2 - 2\langle \psi(P),\psi(Q)\rangle \nonumber\\
&= 1+1-2\,BC(P,Q) \nonumber\\
&= 2\big(1-BC(P,Q)\big).
\end{align}
Plugging into \eqref{eq:app_hellinger_def} yields the identity
\begin{equation}
H^2(P,Q) = 1 - BC(P,Q) = 1 - \sum_{i=1}^m \sqrt{P_iQ_i}.
\label{eq:app_hellinger_bc}
\end{equation}

\subsection{Fisher--Rao distance as spherical geodesic distance}
Recall from Section~\ref{sec:fisherraodist} that  under the square-root embedding, the Fisher--Rao distance on the probability simplex admits
the closed form
\begin{equation}
d_{\mathrm{FR}}(P,Q)
= 2\acos\!\left(\sum_{i=1}^m \sqrt{P_iQ_i}\right)
= 2\acos\!\big(BC(P,Q)\big).  
\label{eq:app_fr_bc}
\end{equation}
Geometrically, $d_{\mathrm{FR}}(P,Q)$ is twice the angle between the unit vectors
$\psi(P)$ and $\psi(Q)$ on the sphere.

Combining \eqref{eq:app_hellinger_bc} and \eqref{eq:app_fr_bc} yields a direct conversion   between Fisher--Rao and Hellinger distances. 
For probability vectors $P,Q$ on $m$ bins,
\begin{equation}
d_{\mathrm{FR}}(P,Q) = 2\acos\!\big(1 - H^2(P,Q)\big).
\label{eq:app_fr_hellinger}
\end{equation}
Equivalently,
\begin{equation}
H^2(P,Q) = 1 - \cos\!\left(\frac{d_{\mathrm{FR}}(P,Q)}{2}\right).
\label{eq:app_hellinger_fr}
\end{equation}

 Indeed, 
by \eqref{eq:app_hellinger_bc}, $BC(P,Q)=1-H^2(P,Q)$. Substituting into
\eqref{eq:app_fr_bc} gives \eqref{eq:app_fr_hellinger}. Rearranging yields~\eqref{eq:app_hellinger_fr}.

\paragraph{Relevance to $\theta$-histograms.}
When $P$ and $Q$ are the normalized histograms estimating $p(\theta\mid A)$ and $p(\theta\mid B)$,
the Bhattacharyya coefficient $BC(P,Q)$ measures histogram overlap in the square-root geometry,
the Hellinger distance measures the corresponding Euclidean separation on the sphere (up to the
$\tfrac{1}{\sqrt{2}}$ factor), and Fisher--Rao measures the intrinsic spherical geodesic distance.
Eq.~\eqref{eq:app_fr_bc}  shows these are equivalent up to a monotone transform.

\end{document}